\documentclass{article}

\RequirePackage{etex}

\usepackage[numbers]{natbib}
\usepackage{algorithm}
\usepackage{algpseudocode}

\usepackage{makecell}

\usepackage[final]{neurips_2024}

\usepackage[utf8]{inputenc} %
\usepackage[T1]{fontenc}    %
\usepackage{hyperref}       %
\usepackage{url}            %
\usepackage{booktabs}       %
\usepackage{amsfonts}       %
\usepackage{nicefrac}       %
\usepackage{microtype}      %
\usepackage{xcolor}         %
\usepackage{booktabs}
\usepackage{wrapfig}
\usepackage{multirow}
\usepackage{makecell}
\usepackage{changdefs}

\newcommand{\appxref}[1]{Appendix~\ref{appx:#1}}

\newcommand{\angstrom}{\text{\r{A}}}

\renewcommand{\ttb}{{\bar{t}}}
\newcommand{\texteq}{\textnormal{eq}}
\newcommand{\pred}{\textnormal{pred}}
\newcommand{\multitask}{\textnormal{multi-task}}
\newcommand{\optim}{\textnormal{optim-cons}}
\newcommand{\score}{\textnormal{score-cons}}

\usepackage{color}

\title{Physical Consistency Bridges Heterogeneous Data in Molecular Multi-Task Learning}

\author{%
  Yuxuan Ren\thanks{Equal contribution.}\;\,\thanks{Work done during an internship at Microsoft Research AI for Science.}, Dihan Zheng$^{*\dagger}$, Chang Liu\thanks{Correspondence to: \textless{}changliu@microsoft.com\textgreater{}},\, %
  Peiran Jin, Yu Shi, Lin Huang, \\[2pt]
  \textbf{Jiyan He$^\dagger$, Shengjie Luo$^\dagger$, Tao Qin, Tie-Yan Liu} \\[6pt]
  Microsoft Research AI for Science %
}

\begin{document}
\abovedisplayskip=4pt
\belowdisplayskip=5pt
\abovedisplayshortskip=3pt
\belowdisplayshortskip=4pt

\maketitle

\begin{abstract}
  In recent years, machine learning has demonstrated impressive capability in handling molecular science tasks. To support various molecular properties at scale, machine learning models are trained in the multi-task learning paradigm.
  Nevertheless, data of different molecular properties are often not aligned: some quantities, \eg equilibrium structure, demand more cost to compute than others, \eg energy, so their data are often generated by cheaper computational methods at the cost of lower accuracy, which cannot be directly overcome through multi-task learning.
  Moreover, it is not straightforward to leverage abundant data of other tasks to benefit a particular task. %
  To handle such data heterogeneity challenges, we exploit the specialty of molecular tasks that there are physical laws connecting them, and design consistency training approaches that allow different tasks to exchange information directly so as to improve one another. %
  Particularly, we demonstrate that the more accurate energy data can improve the accuracy of structure prediction.
  We also find that consistency training can directly leverage force and off-equilibrium structure data to improve structure prediction, demonstrating a broad capability for integrating heterogeneous data.
\end{abstract}

\vspace{-2pt}
\section{Introduction} \label{sec:intro}
\vspace{-2pt}
The field of machine learning has witnessed a blossom of progress in solving molecular science tasks in recent years, including molecular property prediction~\citep{unke2019physnet,ying2021transformers,schutt2021equivariant}, machine-learning force field (energy/force prediction)~\citep{zhang2018deep,gasteiger2021gemnet,chen2022universal,batatia2022mace,musaelian2023learning}, electronic structure~\citep{li2022deep,zhang2024selfconsistency,kirkpatrick2021pushing,remme2023kineticnet,zhang2024overcoming}, molecular structure generation~\citep{noe2019boltzmann,jumper2021highly,pmlr-v162-hoogeboom22a,zheng2024predicting} and design~\citep{watson2023novo,ingraham2023illuminating,gomez2018automatic}.
In molecular research, these tasks are often required jointly: for example, energy prediction is required for molecular stability and dynamics, %
and equilibrium structure (conformation) prediction offers the most probable and characteristic structure for %
understanding molecule interaction and functions.
Multi-task learning is hence adopted, where a model is trained to predict multiple properties using the same number of decoders (output heads) built on a shared encoder (backbone model) (\figref{illustration})~\citep{caruana1997multitask}.
This paradigm is also used for pre-training a model by leveraging as much data as possible that are scattered over various tasks and domains~\citep{zhou2023unimol,zhang2023dpa2,luo2023one}.

Nevertheless, science tasks have some unique challenges beyond conventional machine learning tasks, which cannot be adequately addressed by multi-task learning alone. %
It is more costly to curate a dataset for molecular science tasks since it calls for running physics-theoretic computation algorithms, which come with a stringent accuracy-efficiency trade-off.
Molecular-science datasets are hence generated each with a specific algorithmic portfolio that is economic for the particular purpose. This incurs two challenges regarding data heterogeneity.
\itemone~Different properties in a dataset may come from algorithms in different levels of theory, meaning different levels of accuracy.
This limits the prediction accuracy on some tasks. %
For example, labeling the energy of a molecular structure is yet affordable for common molecules using algorithms in the density functional theory (DFT) level, but producing the equilibrium structure of a molecule requires repeated energy evaluations hence is tens to hundreds of times more costly.
As a result, DFT-level equilibrium structure data are available only in a limited scale~\citep{ramakrishnan2014quantum,nakata2017pubchemqc,hu2021ogb}, while larger-scale datasets~\citep{nakata2020pubchemqc,nakata2023pubchemqc} have to resort to lower levels of theory %
to generate equilibrium structures at scale, which come with a lower level of accuracy.
Using such data, %
multi-task learning alone cannot predict structures in an accuracy higher than the data-generation method.
\itemtwo~Different datasets focus on different tasks, so combining these datasets to enhance the performance on a particular task is not straightforward. %
For example, there are datasets~\citep{chmiela2017machine,chmiela2023accurate,eastman2023spice} that are concerned with force prediction and off-equilibrium structures, which do not provide direct supervision to equilibrium structure prediction.
Although including these additional tasks in multi-task learning could help learn a better encoder (or, representation), this is only based on an empirical observation from a general machine learning perspective and does not directly exploit relevant physical information in the additional tasks. %

\begin{wrapfigure}{r}{.450\textwidth}
  \centering
  \vspace{-16pt}
  \includegraphics[width=.448\textwidth]{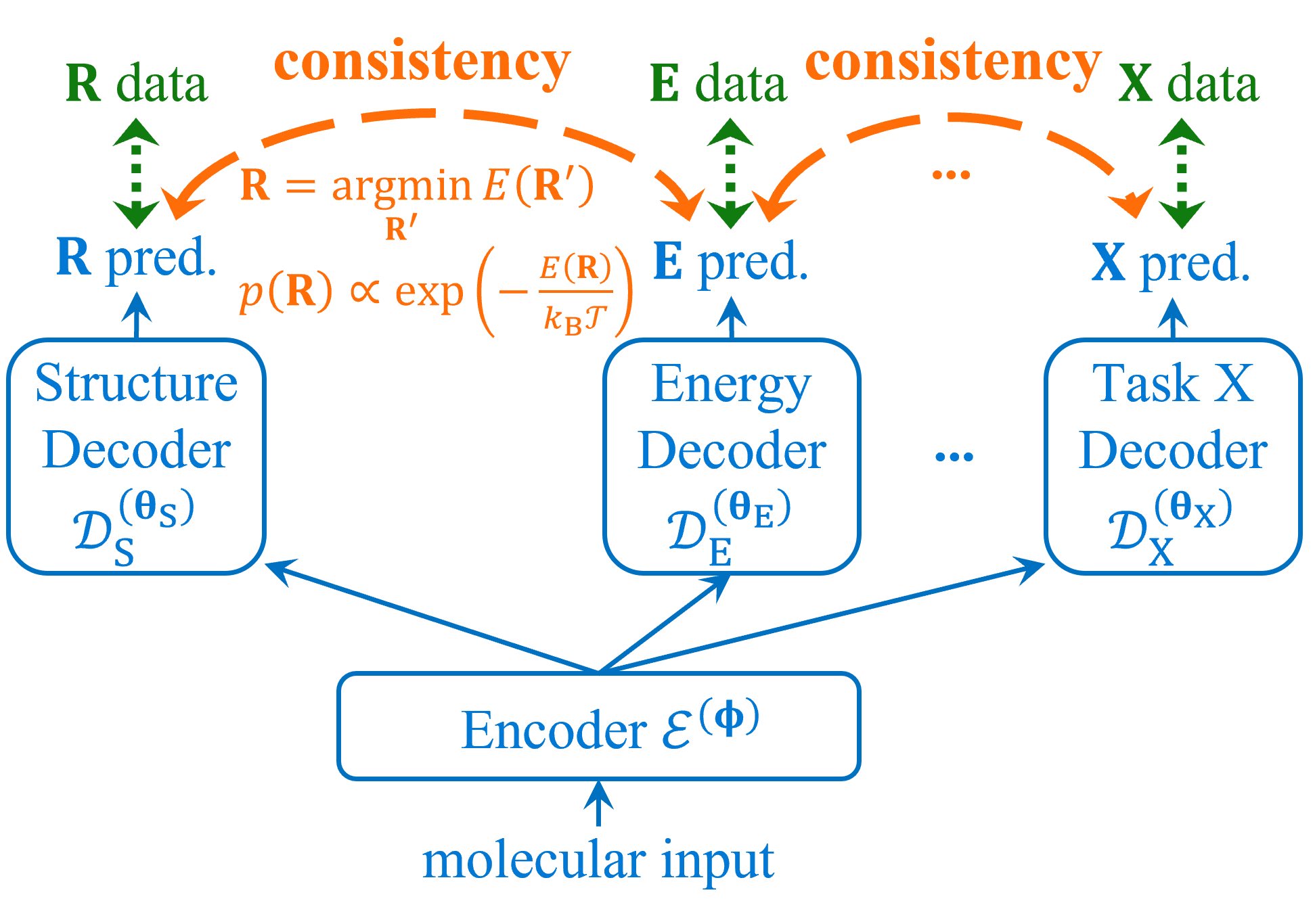}
  \vspace{-18pt}
  \caption{Illustration of the idea of physical consistency. To support multiple tasks (``Task X'' represents a general task), the model (blue solid lines) builds multiple decoders on a shared encoder, which are trained by multi-task learning with data of respective tasks (green dotted double arrows). Physical consistency losses enforce physical laws between tasks (orange dashed double arrows), hence bridge data heterogeneity and directly improve one task from others.
  }
  \vspace{-8pt}
  \label{fig:illustration}
\end{wrapfigure}

In this work, we highlight that science tasks also provide fortunate specialties that come to the rescue.
In contrast to conventional machine learning tasks which are primarily defined by data, science tasks originate in fundamental physical laws, and data are rather the demonstration of such laws. %
These laws impose explicit constraints between tasks, hence define the ``physical consistency'' between model predictions on these tasks.
By enforcing such consistency, model predictions for different tasks are connected and can explicitly share the information in the data of one task to the prediction for other tasks, hence bridging data heterogeneity.
From another perspective, while sharing a common encoder connects various decoders at the input end, physical consistency closes the loop by connecting the decoders at the output end (\figref{illustration}).
With the additional information from the physical consistency, capabilities beyond conventional multi-task learning are enabled:
\itemone~data from a higher-level of theory of one task can improve the accuracy of a physically related task, and
\itemtwo~the abundant data of a physically related task can directly improve the performance of the concerned task.

Concretely, we demonstrate the practical value of the physical consistency between energy prediction and equilibrium structure prediction, two central tasks in molecular science. The consistency can be constructed from two perspectives: the equilibrium structure of a molecule is the structure that attains the minimal energy of the molecule (\secref{optim-csis}), and the equilibrium structure is a sample of the thermodynamic equilibrium distribution at a low temperature (\secref{score-csis}). When adopting the denoising diffusion formulation~\citep{sohl2015deep,ho2020denoising,song2021score} for structure prediction, we show that the two physical laws can be translated into consistency loss functions which connect the energy-prediction model with the structure-prediction model at large and small diffusion steps, respectively.
We apply the consistency losses in the multi-task learning on the PubChemQC B3LYP/6-31G*//PM6 dataset~\citep{nakata2023pubchemqc}, which is perhaps the largest public dataset with DFT-level (B3LYP/6-31G*) energy labels thus highly relevant to model pre-training. %
But its equilibrium structures are generated in the semi-empirical level (PM6), %
which is in a lower level of accuracy.
We use the consistency losses to transmit the information of the higher-level theory in the energy model to the structure model by only optimizing parameters of the structure decoder, %
which achieves the followings.
\itemone~Consistency losses improve structure prediction accuracy beyond multi-task learning alone, when evaluated using DFT-level structures from the PCQM4Mv2~\citep{hu2021ogb} and QM9~\citep{ramakrishnan2014quantum} datasets. %
The advantage persists even after both have undergone finetuning.
\itemtwo~Datasets providing DFT-level force labels on off-equilibrium structures are better leveraged to improve structure prediction accuracy beyond only including force prediction in multi-task learning. Since force is the gradient of the energy, the additional data allow the energy model to learn a better energy landscape, which in turn leads to more accurate structure prediction through the consistency losses. This is a more direct pathway for data on an additional task to benefit structure prediction beyond learning a better representation in multi-task learning.
We remark that the improvement is not from any more accurate structure data, but from bridging data heterogeneity using the ``free lunch'' redeemed from physical laws.

\subsection{Related Work} \label{sec:relw}

\paragraph{Structure prediction.}
In recent years, deep generative models have been used as a powerful tool to generate molecular structures. %
Due to the subtlety that a structure being rotated as a whole is essentially the same structure, early methods opt to generate intermediate geometric variables, such as inter-atomic distances~\citep{pmlr-v119-simm20a,xu2021learning} or torsional angles~\citep{jing2020learning,ganea2021geomol}.
Directly generating Cartesian coordinates of atoms has also been explored where the rotational equivalence is handled by alignment~\citep{zhu2022direct} or leveraging gradient of distances~\citep{pmlr-v139-shi21b,NEURIPS2021_a45a1d12}.
More recently, diffusion models have been used to generate torsional angles~\citep{jing2022torsional}, or atom coordinates~\citep{pmlr-v162-hoogeboom22a,xu2022geodiff} using equivariant models.
Given the generality and superior performance, we adopt an equivariant diffusion model for structure prediction.
While these works may generate multiple meta-stable structures, we aim at the single equilibrium structure for each molecule, and investigate the benefit from an energy model.

\paragraph{Leveraging physical laws between tasks.}
A noticeable example is leveraging the connection between energy and thermodynamic equilibrium distribution to compensate for potentially biased data to better learn the distribution.
The equilibrium distribution in a canonical ensemble is the Boltzmann distribution, which is directly determined by the energy function. From the machine learning perspective, the energy function provides an unnormalized density function of the target distribution. %
Using a flow-based model~\citep{rezende2015variational}, Boltzmann generators~\citep{noe2019boltzmann} inject the energy supervision via the evidence lower bound objective. %
Bootstrapped $\alpha$-divergence objective is thereafter introduced to mitigate mode collapse~\citep{midgley2023flow}.
Zheng et al.~\citep{zheng2024predicting} use a diffusion model %
where the energy supervises the score model at the start diffusion step and is propagated to intermediate steps by enforcing a PDE. %
Similar techniques are also explored for conventional machine learning tasks~\citep{berner2022optimal,vargas2023denoising}.
More recently, Bose et al.~\citep{bose2024iterated} directly connect the energy to intermediate-step scores.
In the opposite direction, Arts et al.~\citep{arts2023two} leverage Boltzmann-distribution samples to learn a coarse-grained energy model. %
In parallel with these works, the current work is devoted to the investigation of leveraging the connection between energy and structure prediction, and is not using an oracle energy function considering the cost of DFT calculation.

\vspace{-2pt}
\section{Technical Background} \label{sec:prelim}
\vspace{-2pt}
Before delving into details of the consistency between energy and structure prediction, we first introduce the problem formulation, and diffusion-based generative formulation for structure prediction.

Chemically, a molecule is specified by the types (\ie, chemical elements) of atoms and bonds, jointly forming a molecular graph $\clG$. A molecule in physical reality can take different structures (conformations) $\bfR \in \bbR^{A \times 3}$, \ie, the collection of 3-dimensional coordinates of its $A$ atoms. %
Many properties of molecule $\clG$ depend on its specific structure $\bfR$, \eg, the (inter-atomic potential) energy, so energy prediction is in the form $E^{(\bftheta)}_\clG(\bfR)$ with model parameters $\bftheta$.

Among possible structures, the equilibrium structure is the one that attains the minimal energy, and is the most representative structure for the molecule. As mentioned, the diffusion formulation has been a preferred choice for equilibrium structure prediction, %
which samples from a distribution $p_\clG(\bfR)$ concentrated at the equilibrium structure.
For this, a primitive structure is sampled from a simple distribution, \eg, the standard Gaussian, which undergoes a diffusion process that transforms the simple distribution to the desired distribution. This is done by reversing the process in the opposite direction, which is easier to construct. For example, %
the process can be taken as the Langevin diffusion that converges %
to standard Gaussian~\citep{song2021score}:
    $\ud \bfR_t = \beta_t \nabla \log \clN(\bfR_t; \bfzro, \bfI) \dd t + \sqrt{\beta_t} \dd \bfW_t
    = -\frac{\beta_t}{2} \bfR_t \dd t + \sqrt{\beta_t} \dd \bfW_t,$
where $\beta_t$ is a time dilation factor~\citep{wibisono2016variational}, and $\bfW_t$ denotes the Wiener process. %
The process starts from the desired distribution $p_{\clG,0} = p_\clG$ at $t=0$ and ends after a sufficiently long period $T$ when the distribution converges, $p_T = \clN(\bfzro, \bfI)$.
The reverse process is known to follow~\citep{anderson1982reverse}:
\begin{align}
    \ud \bfR_\ttb = \frac{\beta_{T-\ttb}}{2} \bfR_\ttb \dd \ttb + \beta_{T-\ttb} \nabla \log p_{\clG,T-\ttb}(\bfR_\ttb) \dd \ttb + \sqrt{\beta_{T-\ttb}} \dd \bfW_\ttb,
    \label{eqn:rev-cont}
\end{align}
where $\ttb := T - t$, which transforms $\clN(\bfzro, \bfI)$ at $\ttb = 0$ to the desired distribution at $\ttb = T$, hence the generation process is constructed.
To simulate the process, the only unknown is the (Fisher's) score function $\nabla \log p_{\clG,t}(\bfR)$ at each diffusion instant, for which a machine-learning model $\bfss^{(\bftheta)}_{\clG,t}(\bfR)$ is introduced. To learn to fit $\nabla \log p_{\clG,t}(\bfR)$, a practical approach is by optimizing the denoising score matching loss~\citep{vincent2011connection}:
    $\bbE_{p_{\clG,0}(\bfR_0)} \bbE_{p(\bfR_t | \bfR_0)} \lrVert*{\bfss^{(\bftheta)}_{\clG,t}(\bfR_t) - \nabla_{\bfR_t} \log p(\bfR_t | \bfR_0)}^2,$
which is convenient since from the Langevin diffusion, we can derive
    $p(\bfR_t | \bfR_0) = \clN(\bfR_t; \sqrt{\alphab_t} \bfR_t, (1-\alphab_t) \bfI),$ %
    where $\alphab_t := \exp(-\int_0^t \beta_s \dd s),$
which is a known distribution. Leveraging the reparameterization trick, %
the loss is reformed as~\citep{song2021score}:
\begin{align}
    \setlength{\abovedisplayskip}{0pt}
    \bbE_{\Unif(t;0,T)} (1-\alphab_t) \bbE_{p_{\clG,0}(\bfR_0)} \bbE_{\clN(\bfeps_t; \bfzro, \bfI)} \lrVert{\bfss^{(\bftheta)}_{\clG,t}(\sqrt{\alphab_t} \bfR_0 + \sqrt{1-\alphab_t} \bfeps_t) + \frac{\bfeps_t}{\sqrt{1-\alphab_t}}}_2^2,
    \label{eqn:dsm-cont}
\end{align}
where the expectation w.r.t $\bbE_{p_{\clG,0}(\bfR_0)}$ can be estimated by averaging over data. %
Once the score model is trained, it can generate structures by replacing $\nabla \log p_{\clG,t}(\bfR)$ and simulating \eqnref{rev-cont}. If only $p_{\clG,0}(\bfR_0)$ is desired (instead of $p_\clG(\bfR_{0:T})$), %
then an equivalent simulation approach can be adopted known as probability-flow ODE~\citep{song2021score}:
\begin{align}
    \ud \bfR_\ttb = \frac{\beta_{T-\ttb}}{2} \big( \bfR_\ttb + \nabla \log p_{\clG,T-\ttb}(\bfR_\ttb) \big) \dd \ttb,
    \label{eqn:ode-cont}
\end{align}
which is equivalent to the deterministic process in denoising diffusion implicit model (DDIM)~\citep{song2021denoising}.

An alternative to the score-prediction formulation is the denoising formulation~\citep{kingma2021variational,daras2024consistent}.
By defining a ``denoising model'' $\bfS^{(\bftheta)}_{\clG,t}(\bfR_t)$ which formulates the score model following:
\begin{align}
    \bfss^{(\bftheta)}_{\clG,t}(\bfR_t) = \frac{\sqrt{\alphab_t} \bfS^{(\bftheta)}_{\clG,t}(\bfR_t) - \bfR_t}{1-\alphab_t},
    \label{eqn:score=denoising}
\end{align}
the training loss function \eqnref{dsm-cont} in terms of $\bfS^{(\bftheta)}_{\clG,t}(\bfR_t)$ becomes:
\begin{align}
    \bbE_t \frac{\alphab_t}{1-\alphab_t} \bbE_{p_{\clG,0}(\bfR_0)} \bbE_{\bfeps_t} \lrVert{\bfS^{(\bftheta)}_{\clG,t}(\sqrt{\alphab_t} \bfR_0 + \sqrt{1-\alphab_t} \bfeps_t) - \bfR_0}_2^2,
    \label{eqn:dsm-denoise-cont}
\end{align}
which follows the intuition to denoise a perturbed structure by predicting the original structure.
This formulation better aligns with the notion of ``structure prediction'', hence can be benefited from successful model architectures~\citep{watson2023novo,ingraham2023illuminating}, and matches structure pre-training strategies~\citep{zaidi2022pre,luo2023one,zhou2023unimol,wei2023diffusion}.

\vspace{-2pt}
\section{Method} \label{sec:method}
\vspace{-2pt}

We begin with the basic formulation of multi-task learning. We then present the two consistency training approaches between energy and structure prediction, based on two physical laws between the two tasks. The approach to directly leveraging physically-related datasets is described at last.

\vspace{-4pt}
\subsection{Multi-Task Learning for Energy and Structure Prediction} \label{sec:multi-task}

Both energy and structure prediction tasks require a comprehensive understanding of the input molecular graph $\clG$ and structure $\bfR$, so a shared encoder $\clE_{\clG,t}^{(\bfphi)}(\bfR)$ with parameters $\bfphi$ is employed.
The time step $t$ is required by the diffusion formulation, which is taken as 0 for energy prediction indicating the input structure is unperturbed and real. %
For energy and structure prediction, the corresponding decoders $\clD_\tnE^{(\bftheta_\tnE)}$ and $\clD_\tnS^{(\bftheta_\tnS)}$ are introduced.
For structure prediction, the denoising formulation is adopted (end of \secref{prelim}).
Under this formulation, the two tasks are handled by:
\begin{align}
    E_\clG^{(\bfphi,\bftheta_\tnE)}(\bfR) = \clD_\tnE^{(\bftheta_\tnE)} \big( \clE_{\clG,t=0}^{(\bfphi)}(\bfR) \big), \quad
    \bfS_{\clG,t}^{(\bfphi,\bftheta_\tnS)}(\bfR) = \clD_\tnS^{(\bftheta_\tnS)} \big( \clE_{\clG,t}^{(\bfphi)}(\bfR),\bfR \big).
    \label{eqn:model-composition}
\end{align}
On one datapoint $(\clG, \bfR, E)$, the multi-task loss is (\cf \eqnref{dsm-denoise-cont}): $ L_\multitask(\bfphi, \bftheta_\tnE, \bftheta_\tnS | \clG, \bfR, E)$
\begin{align}
    = \lambda_\tnE \lrvert{E_\clG^{(\bfphi,\bftheta_\tnE)}(\bfR) - E}
    + \bbE_t \frac{\alphab_t}{1-\alphab_t} \bbE_{\bfeps_t} \lrVert{\bfS^{(\bftheta)}_{\clG,t}(\sqrt{\alphab_t} \bfR + \sqrt{1-\alphab_t} \bfeps_t) - \bfR}_2^2.
    \label{eqn:loss-multi}
\end{align}
A subtlety with the models is geometric invariance and equivariance. Indeed, if a structure $\bfR$ is translated and rotated as a whole, the resulting atom coordinates represent essentially the same structure. The energy and the probability density should keep invariant after the transformation. For energy prediction, this can be guaranteed by using an invariant encoder $\clE_{\clG,t}^{(\bfphi)}(\bfR)$ (no requirement on $\clD_\tnE^{(\bftheta_\tnE)}$). For the probability density, it is known~\citep{kohler2020equivariant} that %
rotational invariance can be achieved by the invariance of $p_{\clG,T}$, which is satisfied by $\clN(\bfzro, \bfI)$, and the equivariance of the %
denoising model $\bfS_{\clG,t}^{(\bfphi,\bftheta_\tnS)}(\bfR)$. For this reason, the input structure $\bfR$ re-enters the structure decoder $\clD_\tnS^{(\bftheta_\tnS)}$, which is implemented with an equivariant architecture.
Translational invariance of density can be achieved by centering the structures; see ref.~\citep{yim2023diffusion} for reasoning.

Nevertheless, multi-task learning is restricted by the level of accuracy of training data. %
As mentioned, %
structure data are often generated in a lower level of accuracy than energy data due to the more demanding nature.
To alleviate this limitation, we exploit physical laws between molecular energy and structure, and propose consistency training losses accordingly to bridge the energy and structure models, thereby enhancing the accuracy of structure prediction from the more accurate energy model.

\subsection{Optimality Consistency} \label{sec:optim-csis}

One direct relationship between energy and equilibrium structure is that the equilibrium structure $\bfR_\texteq$ minimizes the energy, \ie, $\bfR_\texteq = \argmin_{\bfR} E_\clG(\bfR)$. To enforce this optimality condition, we propose an optimality consistency loss $L_\optim$, in the form of ``increase after perturbation'' loss. It is based on the idea that the energy of the equilibrium structure should be lower than that of its perturbed version. %
Denoting $\bfR_\pred^{(\bfphi,\bftheta_\tnS)}$ as the model-predicted equilibrium structure, the loss on a molecule $\clG$ can be written as:
\begin{align}
    L_\optim(\bftheta_\tnS \mid \bfphi, \bftheta_\tnE, \clG) = \bbE_{\bfeta} \max\Big\{ 0, \; E_\clG^{(\bfphi,\bftheta_\tnE)}(\bfR_\pred^{(\bfphi,\bftheta_\tnS)}) - E_\clG^{(\bfphi,\bftheta_\tnE)}(\bfR_\pred^{(\bfphi,\bftheta_\tnS)} + \bfeta) \Big\},
    \label{eqn:optim-csis-general}
\end{align}
where $\bfeta \sim \clN(\bfzro, \Diag(\bfsigma^2))$ is a small perturbation, and each element of $\bfsigma^2$ is independently sampled from $\Unif(0, \sigma_{\text{max}}^2]$. %
Since the purpose is to improve structure prediction accuracy by leveraging the energy model which has seen more accurate labels, so the consistency loss only optimizes the parameters exclusively for the structure prediction utility, \ie, structure decoder parameters $\bftheta_\tnS$. Other parameters $\bfphi$ and $\bftheta_\tnE$ do not optimize this loss. In this way, the consistency loss would not contaminate the energy prediction model with the less accurate structure prediction model.

The standard way to produce $\bfR_\pred^{(\bfphi,\bftheta_\tnS)}$ requires simulating the diffusion process (\eqnref{rev-cont}) or the equivalent ODE (DDIM) (\eqnref{ode-cont}), which calls the denoising model recursively. %
So optimizing $L_\optim$ would involve backpropagation through the simulation process, which can be impractically costly %
and numerically unstable.
Fortunately, we can exploit the intuition in the denoising formulation and find a much cheaper way to generate structure.
The intuition of the denoising model $\bfS_{\clG,t}(\bfR_t)$ is to recover the original structure $\bfR_0$ from the perturbed structure $\bfR_t$. Although this is informationally impossible at the instance level, the denoising model still has a definite learning target at the distributional level, which is %
$\bbE[\bfR_0 | \bfR_t]$~\citep{daras2024consistent}. Particularly, when $t=T$, the correlation between $\bfR_0$ and $\bfR_T$ diminishes, so $\bfS_{\clG,T}(\bfR_T)$ learns to output $\bbE[\bfR_0]$, the expectation of the target distribution, which is the equilibrium structure since the distribution concentrates at that structure. Under this perspective, the model-predicted structure can be generated by
    $\bfR_\pred^{(\bfphi,\bftheta_\tnS)} = \bfS_{\clG,T}^{(\bfphi,\bftheta_\tnS)}(\bfR_T),$
    where $\bfR_T \sim \clN(\bfzro, \bfI).$
This only requires one evaluation of the denoising model.
Nevertheless, the rotational invariance of the structure distribution introduces more subtleties.
\begin{proposition}\label{equal_variant_denoiser}
    Let $\bfS^{(\bftheta)}: \bbR^{A \times 3} \to \bbR^{A \times 3}$ be a rotationally equivariant function; that is, for any rotation matrix $\bfQ \in \mathrm{SO}(3)$ and structure $\bfR \in \bbR^{A \times 3}$, we have $\bfS^{(\bftheta)}(\bfR \bfQ) = \bfS^{(\bftheta)}(\bfR) \bfQ$. Then, for any target structure $\bfR^{\star} \in \bbR^{A \times 3}$, the minimizer of the denoising loss function $L(\bftheta) = \bbE_{\clN(\bfeps; \bfzro, \bfI)} \lrVert*{ \bfS^{(\bftheta)}(\bfeps) - \bfR^{\star} }_2^2$ is the zero map; that is, $\bfS^{(\bftheta)}(\bfR) = \bfzro$, for any $\bfR$.
\end{proposition}
See \appxref{proof_equal_variant_denoiser} for proof.
This conclusion reveals that the learning target of the denoising model at $T$ is trivially all-zero, which cannot serve to generate the equilibrium structure.
To circumvent this, a simple choice is to denoise from a time step $\tau$ that is close but smaller than $T$:
\begin{align}
    \bfR_\pred^{(\bfphi,\bftheta_\tnS)} = \bfS_{\clG,\tau}^{(\bfphi,\bftheta_\tnS)}(\bfeps), \quad
    \bfeps \sim \clN(\bfzro, \bfI). \quad
    \text{(for large $\tau$)}
    \label{eqn:denoise-gen}
\end{align}
The target of the denoising model at $\tau$ is $\bfS_{\clG,\tau}^{(\bfphi,\bftheta_\tnS)}(\bfeps) = \bbE[\bfR_0 | \bfR_\tau = \bfeps]$, which is equivariant w.r.t $\bfR_\tau$. So the input $\bfR_\tau$ provides an orientation information hence breaking the rotational symmetry of the corresponding distribution $p(\bfR_0 | \bfR_\tau)$. The resulting expectation then would not average a structure over orientations evenly, hence not zero.
More explicitly, since $p(\bfR_0 | \bfR_\tau) \propto p(\bfR_0) p(\bfR_\tau | \bfR_0) \propto p(\bfR_0) \exp \left\{-\frac{\| \bfR_0 - \bfR_\tau / \sqrt{\alphab_\tau} \|_2^2}{2(1/\alphab_\tau - 1)} \right\} $, we have $p(\bfR_0 | \bfR_\tau) \propto  p(\bfR_0) \clN(\bfR_0; \bfR_\tau / \sqrt{\alphab_\tau}, (1/ \alphab_\tau - 1) \bfI)$, where the Gaussian factor assigns larger probability along the direction of $\bfR_\tau$, hence breaks the rotational symmetry from $p(\bfR_0)$. %
Under this choice, the optimality consistency loss in \eqnref{optim-csis-general} is specified as: $L_\optim(\bftheta_\tnS \mid \bfphi, \bftheta_\tnE, \clG)$
\begin{align}
    = \bbE_{\bfeta} \bbE_{\bfeps} \max\Big\{ 0, \; E_\clG^{(\bfphi,\bftheta_\tnE)} \big( \bfS_{\clG,\tau}^{(\bfphi,\bftheta_\tnS)}(\bfeps) \big) - E_\clG^{(\bfphi,\bftheta_\tnE)} \big( \bfS_{\clG,\tau}^{(\bfphi,\bftheta_\tnS)}(\bfeps) + \bfeta \big) \Big\}. \quad
    \text{(for large $\tau$)}
    \label{eqn:optim-csis}
\end{align}

\subsection{Score Consistency} \label{sec:score-csis}

The optimality consistency loss only supervises the denoising model at large time steps. For small time steps, an alternative perspective on the physical law between equilibrium structure and energy can help.
Physically, a molecule $\clG$ in a real system can take different structures with different probabilities. When the system is in thermodynamic equilibrium, the probability distribution of the structures can be determined from the energy function of the molecule.
Particularly, in a system with fixed volume and temperature $\clT$, the structure distribution is the well-known Boltzmann distribution,
$p_{\tnB;\clG,\clT}(\bfR) \propto \exp\left(-E_\clG(\bfR) / (k_\tnB \clT) \right)$, where $k_\tnB$ is the Boltzmann constant.
When temperature approaches zero, the distribution becomes concentrated on the equilibrium structure. %
This aligns with the learning target of the diffusion model for equilibrium structure prediction. Through the expression of the Boltzmann distribution, the structure model can thus be connected to the energy model.

To enforce this connection, ideally, the density function $p_\clG^{(\bfphi,\bftheta_\tnS)}(\bfR)$ modeled by the denoising model should match that defined by the energy model. Since the latter only provides an unnormalized density, we enforce their scores to match:
$\bbE_{q(\bfR)} \lrVert*[\big]{ \nabla \log p_\clG^{(\bfphi,\bftheta_\tnS)}(\bfR) + \nabla E_\clG^{(\bfphi,\bftheta_\tnE)}(\bfR) / (k_\tnB \clT) }_2^2$,
where $q(\bfR)$ is a reference distribution.
Nevertheless, it is computationally costly to evaluate the density function from the diffusion model:
$\log p_\clG^{(\bfphi,\bftheta_\tnS)}(\bfR) = \log p_T(\bfR_T^{(\bfphi,\bftheta_\tnS)}) + \int_0^T \frac{\beta_t}2 \frac{\sqrt{\alphab_t}}{1-\alphab_t} \lrparen{3 A \sqrt{\alphab_t} - \nabla \cdot \bfS_{\clG,t}^{(\bfphi,\bftheta_\tnS)}(\bfR_t^{(\bfphi,\bftheta_\tnS)})} \dd t$,
where $\bfR_{t \in [0,T]}^{(\bfphi,\bftheta_\tnS)}$ is the solution to the ODE $\fracdiff{\bfR_t}{t} = \frac{\beta_t}2 \frac{\sqrt{\alphab_t}}{1-\alphab_t} \lrparen{\sqrt{\alphab_t} \bfR_t - \bfS_{\clG,t}^{(\bfphi,\bftheta_\tnS)}(\bfR_t)}$ with initial condition $\bfR_0 = \bfR$~\citep{song2021score}.
Significant computational cost would be incurred from invoking and backpropagating through an ODE solver to evaluate and optimize the density.

We hence turn to another way to leverage this connection.
Note from \eqnref{score=denoising}, the denoising model can be used to recover the score model, which targets the score function of the marginal distribution at the corresponding diffusion time instant. Particularly, the score model at $t=0$ should approximate the score of the desired distribution, which is $p_{\tnB; \clG,\clT}$ for small $\clT$. The energy model can hence provide supervision to the score model by enforcing this connection: $L_\score(\bftheta_\tnS \mid \bfphi, \bftheta_\tnE, \clG)$
\begin{align}
    = \bbE_{p_\tau(\bfR_\tau)} \lrVert*[\bigg]{ \frac{\sqrt{\alphab_\tau} \bfS_{\clG,\tau}^{(\bfphi,\bftheta_\tnS)}(\bfR_\tau) - \bfR_\tau}{1-\alphab_\tau} + \frac{\nabla_{\bfR_\tau} \bfE_\clG^{(\bfphi,\bftheta_\tnE)}(\bfR_\tau)}{k_\tnB \clT} }_2^2. \quad
    \text{(for small $\tau$)}
    \label{eqn:score-csis}
\end{align}
In this score consistency loss, we have avoided taking the time step $\tau$ exactly zero for numerical stability consideration. Indeed, when $\tau \to 0$, the denoising model $\bfS_{\clG,\tau}^{(\bfphi,\bftheta_\tnS)}$ approaches the identity map, and $\alphab_\tau$ approaches 1, making the first term in \eqnref{score-csis} an indeterminate form of type 0/0, which may render numerical stability issues.
For the reference distribution to generate data to evaluate the loss, one can choose either the data distribution $p_{\clG,0}$ for which the energy model gives more confident results, or the perturbed distribution $p_{\clG,\tau}$ (can be sampled by adding noise to a data sample $\bfR_0$ following $p(\bfR_t | \bfR_0)$) for relevance to how the denoising model $\bfS_{\clG,\tau}^{(\bfphi,\bftheta_\tnS)}$ is invoked. Through some trials, we found the latter gives slightly better results. %

Finally, as is the case for the optimality consistency loss, the score consistency loss also only optimizes the parameters $\bftheta_\tnS$ of the structure decoder, to ensure the energy model would not be misled by the less accurate structure data.
To implement this unconventional optimization requirement, we list detailed algorithms for the two consistency losses in \appxref{consis-alg} in terms of the actual model components $\clE_{\clG,t}^{(\bfphi)}$, $\clD_\tnE^{(\bftheta_\tnE)}$, and $\clD_\tnS^{(\bftheta_\tnS)}$.

\subsection{Leveraging Physically-Related Data} \label{sec:data-fusion}

Besides the difference in the level of theory to generate data, the heterogeneity of molecular-science datasets also lies in the difference of concerned quantities. Compared to the enormous chemical space, available datasets for equilibrium structure are still not abundant, while there is a vast amount of data generated for physically related but different tasks, for example, labels of atomic forces, and data on off-equilibrium structures.
We highlight that the consistency losses can leverage such datasets in an explicit way to further improve equilibrium structure prediction.
Note that the consistency losses \eqnsref{optim-csis,score-csis} works by offering the information at a higher level of theory in the energy model, in the form of \emph{energy landscape} on the structure space; \ie, ranking different structures in optimality consistency, and providing energy gradient in score consistency.
For better learning the landscape, the force labels, which are negative gradients of the energy, provides first-order information of the landscape, and energy and force labels on multiple off-equilibrium structures enable better exploration on the structure space. %
This approach provides a more direct and concrete information path to equilibrium structure prediction than helping learn a better representation in multi-task learning.
For learning a better energy landscape, the force labels are used to directly supervise the gradient of the energy model. The loss term for a datapoint $(\clG,\bfR,\bfF)$ is:
\begin{align}
    L_\textnormal{force}(\bfphi, \bftheta_\tnE \mid \bfR, \bfF) = \lrVert*[\big]{\nabla_\bfR E_\clG^{(\bfphi,\bftheta_\tnE)}(\bfR) + \bfF}_2^2,
    \label{eqn:loss-force}
\end{align}
where $\bfF$ is the force label. Note that there may be multiple $(\bfR,\bfF)$ data pairs for one molecule $\clG$, which provide even richer information on the energy landscape.

\vspace{-2pt}
\section{Experiments} \label{sec:expm}
\vspace{-2pt}
In this section, we demonstrate the advantages of incorporating the proposed consistency losses into multi-task learning. %
Implementation details are provided in \appxref{imple_details}.

\subsection{Setup} \label{sec:exp_setup}

\paragraph{Datasets.}
We consider multi-task learning of energy and structure prediction on the PubChemQC B3LYP/6-31G*//PM6 dataset~\citep{nakata2023pubchemqc} (abbreviated as PM6), which is seemingly the largest ($\sim$86M molecules) public available dataset with DFT-level property labels, hence a preferred setting for pre-training a molecular model.
The energy labels are in the DFT (B3LYP/6-31G*) level, while the equilibrium structures are produced at the semi-empirical PM6~\citep{stewart2007optimization} level, which is less accurate than DFT.
Consistency training is hence considered to improve structure prediction accuracy using the more accurate energy data.
To evaluate the effect of improved structure prediction accuracy beyond the PM6 level, the accuracy is evaluated against structures generated at the DFT level, which are available in the PCQM4Mv2 dataset~\citep{hu2021ogb} (abbreviated as PCQ) and the QM9 dataset~\citep{ramakrishnan2014quantum}.

\vspace{-2pt}
\paragraph{Evaluation.}
Each of the evaluation datasets of PCQ and QM9 is spilt into three disjoint sets for training, validation, and test. The training and validation sets are for optional fine-tuning (see \secref{finetune}). Following existing convention~\citep{pmlr-v139-shi21b,xu2022geodiff,zhu2022direct}, each test set is prepared by uniformly randomly selecting 200 distinct molecules from PCQ or QM9 that do not appear in the training dataset (PM6), which already makes the test molecules sufficiently dissimilar from training molecules (\appxref{dissimilar}).

On each test molecule, we sample 200 structures using the model, calculate their rooted mean square deviations (RMSDs) against the equilibrium structure in the test set, and evaluate the mean and the minimum over these RMSDs.
Due to the geometric invariance of the structure distribution, the RMSD is evaluated after translational and rotational alignment of two structures using the Kabsch algorithm~\citep{kabsch1976solution}.
We consider both the denoising (\eqnref{denoise-gen}) and the DDIM (\eqnref{ode-cont}) approaches for structure sampling.
We also provide coverage evaluation results in \appxref{coverage}.

In each setting, we independently repeat the evaluation process for five times using different random seeds, and report the mean of the repeats in the following tables. The standard deviations and t-test p-values are collectively provided in \appxref{std}.
In settings using consistency training, both the optimality (\eqnref{optim-csis}) and score consistency losses (\eqnref{score-csis}) are added to the multi-task training loss (\eqnref{loss-multi}).
Validation results for training in terms of both energy prediction and structure generation are provided in \appxref{validation}.

\subsection{Structure Prediction Results} \label{sec:wo_force}

\begin{table}[b]
\centering
\caption{Test RMSD (\angstrom; lower is better) of structure prediction by multi-task learning and consistency learning on PM6 dataset.}
\label{tab:vanilla_test}
\begin{tabular}{rcccccccc}
\toprule
Test Set & \multicolumn{4}{c}{PCQ} & \multicolumn{4}{c}{QM9} \\ \cmidrule(lr){2-5} \cmidrule(lr){6-9}
Generated by & \multicolumn{2}{c}{Denoising} & \multicolumn{2}{c}{DDIM} & \multicolumn{2}{c}{Denoising} & \multicolumn{2}{c}{DDIM} \\
\cmidrule(lr){2-3} \cmidrule(lr){4-5} \cmidrule(lr){6-7} \cmidrule(lr){8-9}
Struct. Stat. & Mean & Min & Mean & Min & Mean & Min & Mean &  Min  \\ \midrule
Multi-Task & 1.189 & 0.655 & 1.041 & 0.361 & 0.928 & 0.545 & 0.669 & 0.197 \\
Consistency & \textbf{1.158} & \textbf{0.645} & \textbf{1.007} & \textbf{0.346 }& \textbf{0.848} & \textbf{0.490 }& \textbf{0.650 }&\textbf{ 0.194} \\ \bottomrule
\end{tabular}
\end{table}
We first evaluate the effect of consistency training over multi-task training following the above settings.
The results are shown in \tabref{vanilla_test}.
We see that consistency training enhances the accuracy of structure prediction beyond multi-task learning consistently, without using any more accurate structure data.
The improvement is significant when compared to the standard deviations provided in Appendix \tabref{std_pre}, which are as low as around 0.003~$\angstrom$ (Appendix \tabref{t-test} verifies t-test significance).
We note that this improvement is not at the cost of a lower energy prediction accuracy, as indicated by the energy prediction results in Appendix \tabref{energy_valid}(left). %

To further examine that the improvement is from the effect of consistency training, we evaluate the energy of the predicted structure and the true DFT-level equilibrium structure in the PCQ dataset for each test molecule using the model, which is shown in the scatter plot of \figref{energy_ana}(left).
We observe that when using consistency training, the overall energy is reduced, indicating that the predicted structures indeed have lower energy hence closer to the true DFT-level equilibrium structure.
To quantify this improvement, in \appxref{egap} we define an ``energy gap'' metric, and the results shown in \tabref{egap} consolidate the observation.

\begin{figure}[t]
\centering
\begin{tabular}{ccc}
\includegraphics[width=.3\textwidth]{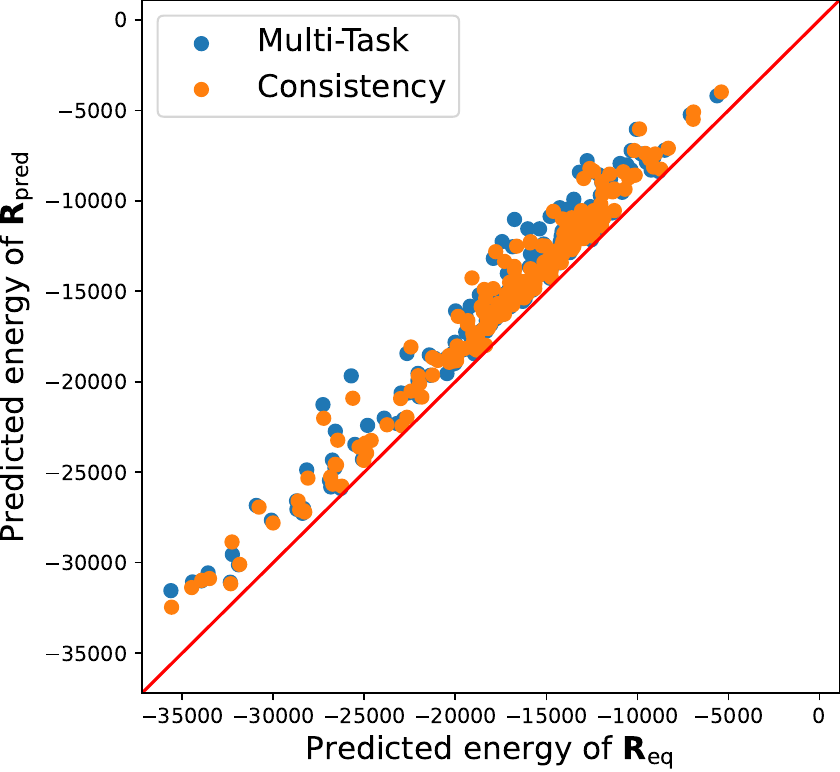} &
\includegraphics[width=.3\textwidth]{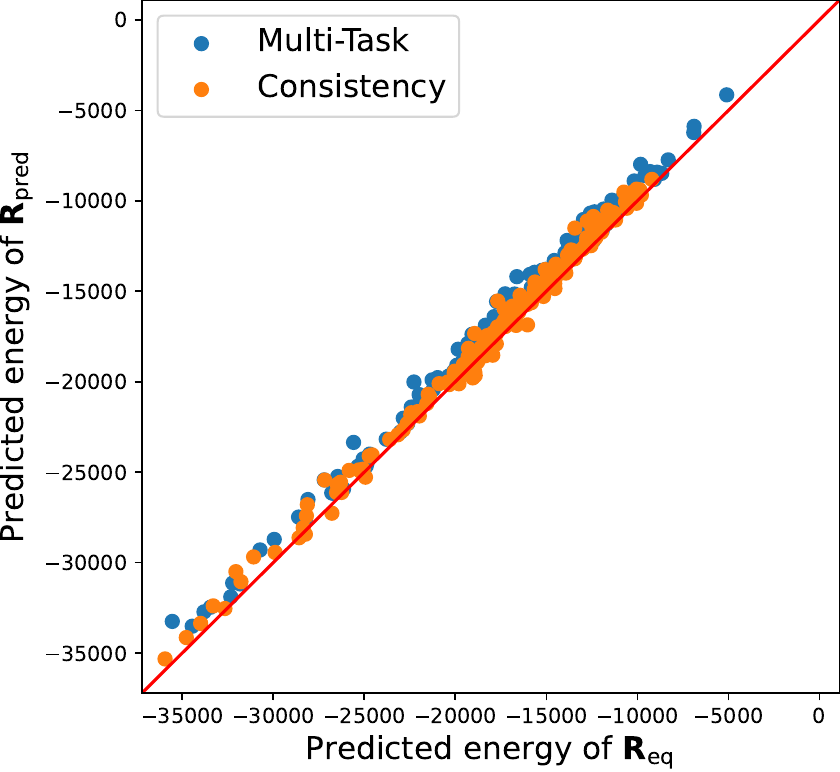} &
\includegraphics[width=.3\textwidth]{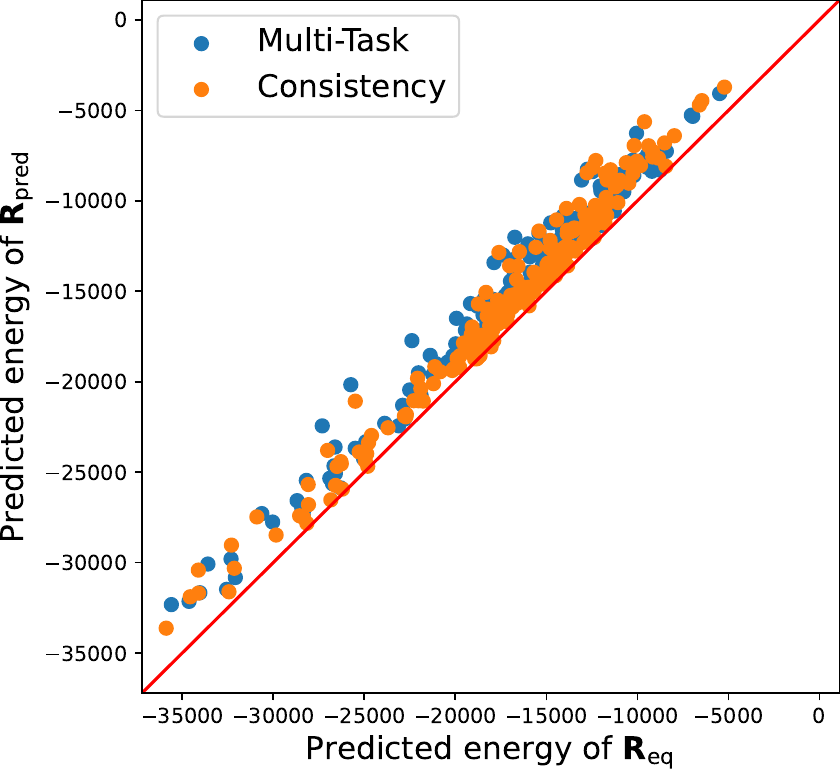} \\
\end{tabular}
\caption{Comparison of energy (eV) on the model-generated structure $\bfR_{\text{pred}}$ using the denoising method and the equilibrium structure $\bfR_{\text{eq}}$ in the PCQ dataset. Each point represents the model-predicted energy values on the two structures for one test molecule. Models are trained on (\textbf{left}) the PM6 dataset, (\textbf{middle}) the PM6 dataset and SPICE force dataset, and (\textbf{right}) the PM6 dataset with a subset of force labels. The closer a point lies to the diagonal line, the closer the energy of the predicted structure is to the minimum energy, indicating a closer prediction of equilibrium structure.
}
\label{fig:energy_ana}
\end{figure}

\vspace{-2pt}
\subsection{Results using Physically-Related Data} \label{sec:force_fusion}
\vspace{-2pt}

We next investigate the effect of consistency training for leveraging physically related data, as explained in \secref{data-fusion}. For this, we consider the force data in the SPICE dataset (PubChem subset)~\citep{eastman2023spice}. The force data are also available on multiple off-equilibrium structures for each molecule.
We note the subtlety that the setting of DFT in SPICE, $\omega$B97M-D3(BJ)/def2-TZVPPD, is different from that for generating energy labels in the PM6 dataset. (Up to our knowledge, there does not seem to exist a public force dataset that matches the DFT setting as the PM6 dataset.)
On one hand, calculation on near-equilibrium structures of small molecules is not very sensitive to DFT settings, especially for force calculation (even less affected than energy), while the force labels on multiple structures could be more valuable to learning the energy landscape despite the mismatch. So we still consider it as a relevant investigation setting. Note energy labels in SPICE are not used, and additional structures are not used for training the structure decoder.
On the other hand, to reduce the gap, we also generated in-house force labels on a subset of PM6 structures using the same DFT setting (B3LYP/6-31G*) as for the PM6 dataset energy labels. The systematic error is controlled, although for each molecule there is only one labeled structure.

The results after incorporating SPICE force data and PM6 subset force data in training are shown in \tabref{pretrain-spice-and-pm6-force}.
We first observe that consistency training still outperforms multi-task training in structure prediction in all cases (see Appendix Tables~\ref{tab:std_pre} and~\ref{tab:t-test} for standard derivations and t-test p-values).
We note that this improvement is not at the cost of a lower energy prediction accuracy, as indicated by Appendix \tabref{energy_valid}(middle) indicates energy prediction is also not compromised in consistency training in this case.
When compared to \tabref{vanilla_test}, we see that the inclusion of force data does not uniformly enhance multi-task learning performance, since the mechanism to learn a better representation is still implicit and indirect and may require extensive tuning. %
In contrast, using consistency losses improves structure prediction more consistently when physically related data are available in training.
These observations indicate that consistency loss training can potentially assist the model in more effectively utilizing data from different sources or modalities.

Energy analysis is presented in \figref{energy_ana}(middle) utilizing SPICE force data, and in \figref{energy_ana}(right) with force labels on a subset of PM6 molecules. The data indicate that training with consistency loss results in lower predicted energies than multi-task training. Furthermore, we observe that the structures predicted by models trained with force datasets (\figref{energy_ana}, middle and right) have lower predicted energies compared to those trained exclusively on the PM6 dataset (\figref{energy_ana}, left), illustrating the advantage of incorporating force labels.

\begin{table}[t]
\setlength{\tabcolsep}{5pt}
\centering
\caption{Test RMSD (\angstrom; lower is better) of structure prediction by multi-task learning and consistency learning on the PM6 dataset \emph{with additional SPICE force dataset or PM6 subset force data}.}
\label{tab:pretrain-spice-and-pm6-force}
\begin{tabular}{r|rcccccccc}
\toprule
\multirow{3}{*}{\makecell[cr]{Additional \\ Training Data}}
& Test Set & \multicolumn{4}{c}{PCQ} & \multicolumn{4}{c}{QM9} \\ \cmidrule(lr){3-6} \cmidrule(lr){7-10}
& Generated by & \multicolumn{2}{c}{Denoising} & \multicolumn{2}{c}{DDIM} & \multicolumn{2}{c}{Denoising} & \multicolumn{2}{c}{DDIM} \\
\cmidrule(lr){3-4} \cmidrule(lr){5-6} \cmidrule(lr){7-8} \cmidrule(lr){9-10}
& Struct. Stat. & Mean & Min & Mean & Min & Mean & Min & Mean &  Min \\
\midrule
\multirow{2}{*}{\makecell[cr]{SPICE \\ force}}
& Multi-Task & 1.161 & 0.631 & 1.047 & 0.373 & 0.876 & 0.486 & 0.670 & 0.207 \\
& Consistency & \textbf{1.147} & \textbf{0.590} & \textbf{1.013} & \textbf{0.345} & \textbf{0.842} & \textbf{0.485} & \textbf{0.644} & \textbf{0.194} \\
\midrule
\multirow{2}{*}{\makecell[cr]{PM6 subset \\ force}}
& Multi-Task & 1.199 & 0.672 & 1.027 & 0.365 & 0.914 & 0.545 & 0.648 & 0.193 \\
& Consistency & \textbf{1.113} & \textbf{0.629} & \textbf{1.019} & \textbf{0.351} & \textbf{0.836} & \textbf{0.488} & \textbf{0.646} & \textbf{0.192}  \\ \bottomrule
\end{tabular}
\end{table}

\begin{table}[b]
\centering
\caption{Test RMSD (\angstrom; lower is better) \emph{after finetuning} for structure prediction pre-trained by multi-task learning and consistency learning on the PM6 dataset.}
\label{tab:finetune_test}
\begin{tabular}{rcccccccc}
\toprule
Test Set & \multicolumn{4}{c}{PCQ} & \multicolumn{4}{c}{QM9} \\ \cmidrule(lr){2-5} \cmidrule(lr){6-9}
Generated by & \multicolumn{2}{c}{Denoising} & \multicolumn{2}{c}{DDIM} & \multicolumn{2}{c}{Denoising} & \multicolumn{2}{c}{DDIM} \\
\cmidrule(lr){2-3} \cmidrule(lr){4-5} \cmidrule(lr){6-7} \cmidrule(lr){8-9}
Struct. Stat. & Mean & Min & Mean & Min & Mean & Min & Mean &  Min  \\ \midrule
Multi-Task & 1.158 & 0.614 & 0.921 & 0.220 & 0.889 & 0.467 & 0.501 & 0.090\\
Consistency & \textbf{1.152} & \textbf{0.610} & \textbf{0.918} & \textbf{0.218} & \textbf{0.835} & \textbf{0.420} & \textbf{0.493} & \textbf{0.076} \\ \bottomrule
\end{tabular}
\end{table}

\subsection{Fine-Tuning Results} \label{sec:finetune}

In addition to the zero-shot prediction evaluations, we further fine-tune the pre-trained models investigated in \secref{wo_force} using DFT-level structures in PCQ or QM9 training datasets. The results presented in \tabref{finetune_test} indicate that the inclusion of consistency loss in pre-training still enhances the accuracy of structure prediction even after fine-tuning. See Appendix Tables~\ref{tab:std_fine} and~\ref{tab:t-test} for statistical significance.
This could be attributed to that using consistency losses in pre-training can already inform the model of more accurate structure information, leading to a state that is more relevant to the underlying physics, which is a favored starting point for further improvements through fine-tuning. Results of fine-tuning models that are pre-trained with SPICE force and PM6 subset force are shown in \appxref{finetune_force}, which further demonstrates the advantages of the consistency loss.

\vspace{-2pt}
\section{Conclusion and Discussion} \label{sec:concl}
\vspace{-2pt}
This work leverages physical laws between molecular tasks to bridge data heterogeneity in multi-task learning.
Consistency losses are designed to enforce physical laws between inter-atomic potential energy prediction and equilibrium structure prediction.
They have shown to improve structure prediction beyond the typical accuracy level of structure data by leveraging abundant energy data in a higher level of accuracy, and can directly leverage force and off-equilibrium structure data to further improve the accuracy. The advantage still holds after finetuning. %
We would like to highlight that the improvement comes ``for free'' as no additional data (\eg, more accurate structure data) are required, demonstrating the value of physical laws in learning molecular tasks.
The idea bears broader generality as data heterogeneity is ubiquitous in the science domain, and data for a specific task are often limited in either abundancy or accuracy.

The current work is limited to the consistency between energy and structure prediction, while more consistency laws can be considered in molecular science. Apart from mentioned works in connecting energy and thermodynamic distribution, more possibilities include electronic structure and molecular properties, and fine-grained and coarse-grained structures and macroscopic statistics.
The significance of improvement in this work is still limited by the abundance of the data involved. Further improvement can be expected with more abundant/diverse but possibly less accurate structure data from, \eg, RDKit~\citep{landrum2013rdkit} or experimental measurements, or energy/force datasets in matching level of theory that more extensively explore the structural space.

\bibliography{main}
\bibliographystyle{unsrtnat}

\newpage
\appendix
\section*{Appendix}

\numberwithin{equation}{section}
\numberwithin{figure}{section}
\numberwithin{table}{section}
\numberwithin{algorithm}{section}
\numberwithin{theorem}{section}

\section{Proof of Proposition~\ref{equal_variant_denoiser}} \label{appx:proof_equal_variant_denoiser}

We first prove the following two preliminary Lemmas.

\begin{lemma}\label{gaussion_rotation_invariant}
    Let $\bfeps = \left[ \bfeps_1\trs,\dots,\bfeps_A\trs \right]\trs$, and $\bfeps_i$ are independent, identically distributed random vectors with $\bfeps_i \sim \clN(\bfzro, \bfI_3)$. $\bfQ$ is a random variable uniformly distributed over $\mathrm{SO}(3)$. Then the random matrix $\bfgamma = \bfeps \bfQ$ has the same distribution as $\bfeps$.
\end{lemma}

\begin{proof}
    Consider the probability density function of $\bfgamma$:
    \begin{align}
        p_{\bfgamma}(\bfgamma) = \int p_{\bfgamma,\bfQ}(\bfgamma, \bfQ) \tndd \bfQ = \int p_{\bfQ}(\bfQ) p_{\bfgamma \mid \bfQ}(\bfgamma \mid \bfQ) \tndd \bfQ.
    \end{align}
    Since $\bfgamma = \bfeps \bfQ$, we have:
    \begin{align}
        p_{\bfgamma \mid \bfQ}(\bfgamma \mid \bfQ) = \prod_{i=1}^A \frac{1}{(2\pi)^{3/2}} \exp\left\{ -\frac{\|\bfQ\trs \bfgamma_i \|_2^2}{2} \right\} = \frac{1}{(2\pi)^{3/2}} \prod_{i=1}^A \exp\left\{ -\frac{\| \bfgamma_i \|_2^2}{2} \right\} = p_{\bfeps}(\bfgamma).
    \end{align}
    Thus, it follows that:
\begin{align}
        p_{\bfgamma}(\bfgamma) = \int p_{\bfQ}(\bfQ) p_{\bfgamma \mid \bfQ}(\bfgamma \mid \bfQ) \tndd \bfQ = \int p_{\bfQ}(\bfQ) p_{\bfeps}(\bfgamma) \tndd \bfqq = p_{\bfeps}(\bfgamma),
    \end{align}
    which demonstrates that $\bfgamma$ is identically distributed as $\bfeps$.
\end{proof}

\begin{lemma}\label{mean_of_SO3}
    Assume $\bfQ$ is a random variable uniformly distributed over $\mathrm{SO}(3)$. Then $\bbE_{\bfQ} \bfQ = \bfzro$.
\end{lemma}

\begin{proof}
    For any fixed $\bfQ_0 \in \mathrm{SO}(3)$, the distribution of $\bfQ \bfQ_0$ is identical to that of $\bfQ$ due to the uniformity of the distribution over the group $\mathrm{SO}(3)$. Consequently, we have:
    \begin{align}
        \bbE_{\bfQ} \bfQ = \bbE_{\bfQ} \bfQ \bfQ_0 = \left[\bbE_{\bfQ} \bfQ \right] \bfQ_0
    \end{align}
    Suppose $\bbE_{\bfQ} \bfQ = [\bfqq_1, \bfqq_2, \bfqq_3]$, where $\bfqq_1, \bfqq_2, \bfqq_3$ are the columns of $\bbE_{\bfQ} \bfQ$. Then, it follows that:
    \begin{align}
        [\bfqq_1, \bfqq_2, \bfqq_3] = [\bfqq_1, \bfqq_2, \bfqq_3]\bfQ_0.
    \end{align}
    Selecting $\bfQ_0 = \Diag\{-1, -1, 1\}$ yields $\bfqq_1 = - \bfqq_1$ and $\bfqq_2 = - \bfqq_2$, , which implies that $\bfqq_1 = \bfqq_2 = \bfzro$. Similarly, choosing $\bfQ_0 = \Diag\{1, -1, -1\}$ results in $\bfqq_3 = - \bfqq_3$, hence $\bfqq_3 = \bfzro$. Therefore, we conclude that $\bbE_{\bfQ} \bfQ = \bfzro$.
\end{proof}

We now proceed to prove Proposition~\ref{equal_variant_denoiser}. Let $\bfQ$ be a random variable uniformly distributed over $\mathrm{SO}(3)$. Then according to Lemma~\ref{gaussion_rotation_invariant}, the random matrix $\bfgamma = \bfeps \bfQ$ is identically distributed as $\bfeps$. Then we have:
\begin{align}
    L(\bftheta) &= \bbE_{\bfeps} \| \bfS^{(\bftheta)}(\bfeps) - \bfR^{\star} \|_2^2 = \bbE_{\bfgamma} \| \bfS^{(\bftheta)}(\bfgamma) - \bfR^{\star} \|_2^2 = \bbE_{\bfeps,\bfQ} \| \bfS^{(\bftheta)}(\bfeps \bfQ) - \bfR^{\star} \|_2^2 \\
    & = \bbE_{\bfeps,\bfQ} \| \bfS^{(\bftheta)}(\bfeps) \bfQ - \bfR^{\star} \|_2^2 = \bbE_{\bfeps} \left[\bbE_{\bfQ} \| \bfS^{(\bftheta)}(\bfeps) - \bfR^{\star} \bfQ\trs \|_2^2\right].
\end{align}
For any $\bfeps \in \bbR^{A \times 3}$, the objective function $\bbE_{\bfQ} \lrVert{\bfS^{(\bftheta)}(\bfeps) - \bfR^{\star} \bfQ\trs }_2^2$ achieves its minimum when $\bfS^{(\bftheta)}(\bfeps) = \bbE_{\bfQ} [\bfR^{\star} \bfQ\trs] = \bfR^{\star} [\bbE_{\bfQ} \bfQ\trs]$. Since $\bbE_{\bfQ} \bfQ\trs = \bbE_{\bfQ} \bfQ$ and using Lemma~\ref{mean_of_SO3}, we have $\bbE_{\bfQ} \bfQ\trs = \bfzro$, which implies $\bfS^{(\bftheta)}(\bfeps) = \bfzro$. Therefore, $\bfS^{(\bftheta)}$ is a zero map.

\section{Implementation Details} \label{appx:imple_details}

\subsection{Implementation Details for Consistency Losses} \label{appx:consis-alg}
The proposed consistency loss $L_\optim$ in \eqnref{optim-csis} and $L_\score$ in \eqnref{score-csis} are designed to update the parameters $\bftheta_\tnS$ within the structure prediction model $\bfS_{\clG,\tau}^{(\bfphi,\bftheta_\tnS)}$. To achieve this, we employ the stop gradient operation $ \mathrm{SG}(\cdot)$ to prevent unnecessary gradient computation for the parameters $\bfphi$  in the encoder model $\clE_{\clG,t}^{(\bfphi)}$ and $\bftheta_\tnE$ in the energy decoder $\clD_\tnE^{((\bftheta_\tnE)}$. The detailed implementations of optimality consistency and score consistency are presented in Alg.~\ref{alg:optim-csis} Alg.~\ref{alg:score-csis}, respectively.

\begin{algorithm}
\caption{Implementation of Optimality Consistency Loss}
\begin{algorithmic}[1]
\Require Encoder: $\clE_{\clG,t}^{(\bfphi)}(\bfR)$, structure decoder $\clD_\tnE^{(\bftheta_\tnE)}$, energy decoder $\clD_\tnS^{(\bftheta_\tnS)}$, Diffusion time step $\tau$.
\Ensure $\nabla_{\bftheta_\tnS} L_\optim$
\State Sample $\bfeps$ and $\bfeta$.
\State Extract the molecular feature $\bfhh \gets \clE_{\clG,t=\tau}^{(\bfphi)}(\bfeps)$ using the encoder.
\State Apply the stop gradient operation to the molecular feature and obtain $\bar{\bfhh} \gets \mathrm{SG} (\bfhh)$ through Pytorch's \texttt{.detach()} method.
\State Compute the denoised structure $\hat{\bfR}_{0} \gets \clD_\tnS^{(\bftheta_\tnS)}(\bar{\bfhh}, \bfeps)$ using the structure decoder.
\State Set \texttt{requires\_grad = False} for the parameters in the energy model $E_\clG^{(\bfphi,\bftheta_\tnE)}(\bfR)=\clD_\tnE^{(\bftheta_\tnE)} \big( \clE_{\clG,t=0}^{(\bfphi)}(\bfR) \big)$
\State Evaluate the loss $L_\optim = \max \left\{ 0, \; E_\clG^{(\bfphi,\bftheta_\tnE)} ( \hat{\bfR}_{0}) -  E_\clG^{(\bfphi,\bftheta_\tnE)} ( \hat{\bfR}_{0} + \bfeta ) \right\}.$
\State Determine the gradient $\nabla_{\bftheta_\tnS} L_\optim$ through automatic differentiation.
\end{algorithmic}
\label{alg:optim-csis}
\end{algorithm}

\begin{algorithm}
\caption{Implementation of Score Consistency Loss}
\begin{algorithmic}[1]
\Require Encoder: $\clE_{\clG,t}^{(\bfphi)}(\bfR)$, structure decoder $\clD_\tnE^{(\bftheta_\tnE)}$, energy decoder $\clD_\tnS^{(\bftheta_\tnS)}$, Diffusion time step $\tau$.
\Ensure $\nabla_{\bftheta_\tnS} L_\score$
\State Sample $\bfR_\tau$.
\State Extract the molecular feature $\bfhh_0 \gets  \clE_{\clG,t=0}^{(\bfphi)}(\bfR_\tau)$ using the encoder.
\State Compute the free energy $\bfE \gets \clD_\tnE^{(\bftheta_\tnE)} \big( \bfhh_0 \big)$ with the energy decoder.
\State Compute the energy gradient $\bfF \gets \nabla_{\bfR_\tau} \bfE$ using PyTorch's \texttt{torch.autograd}.
\State Apply the stop gradient operation (\texttt{.detach()} method in Pytorch) to the energy gradient: $\bar{\bfF} \gets \mathrm{SG}\left( \bfF \right)$.
\State Extract the molecular feature $\bfhh_\tau \gets \clE_{\clG,t=\tau}^{(\bfphi)}(\bfR_\tau)$ using the encoder.
\State Obtain $\bar{\bfhh}_\tau \gets \mathrm{SG} (\bfhh_\tau)$ using \texttt{.detach()}.
\State Compute the denoised structure $\hat{\bfR}_0 \gets  \clD_\tnS^{(\bftheta_\tnS)} \big( \mathrm{SG} (\bar{\bfhh}_\tau),\bfR_\tau \big)$ with structure decoder.
\State Evaluate the loss $L_\score = \lrVert*[\bigg]{ \frac{\sqrt{\alphab_\tau} \hat{\bfR}_0 - \bfR_\tau}{1-\alphab_\tau} + \frac{ \bar{\bfF} }{k_\tnB \clT} }_2^2$.
\State Determine the gradient $\nabla_{\bftheta_\tnS} L_\score$ through automatic differentiation.
\end{algorithmic}
\label{alg:score-csis}
\end{algorithm}

\subsection{Model Architecture}

The model is composed of an encoder, and two decoders for structure and energy prediction.

The encoder $\clE_{\clG,t}^{(\bfphi)}(\bfR)$ is a simple modification to the Graphormer model~\citep{ying2021transformers,shi2022benchmarking}, which additionally adopts diffusion time $t$ embedding into both node features and pairwise-distance attention bias. The encoder consists of 24 layers of Graphormer, with the dimension of both hidden and feed-forward layers set to 768. It utilizes a multi-head attention mechanism with 32 heads and employs 128 Gaussian Basis kernels for enhancing the positional encoding.

For the time embedding, we implement a \texttt{SinusoidalPositionEmbeddings} module and a \texttt{TimeStepEncoder} module. The former generates time-dependent sinusoidal embeddings, and the latter refines these embeddings using a feed-forward network with a GELU activation function. The resulting time embeddings are then integrated into the node features to inform the model of temporal information.

Additionally, we incorporate time embeddings into the attention mechanism by computing a structure-based attention bias. This is achieved by calculating the outer product of the time embeddings and using the result as an additive bias in the self-attention layers. This integration allows the model to adapt its attention based on the temporal relationships between nodes in the graph.

On top of the encoder, the energy decoder $\clD_\tnE^{(\bftheta_\tnE)}(\bfhh)$ is a simple MLP layer concatenated to the invariant node features $\bfhh$.

The structure decoder $\clD_\tnS^{(\bftheta_\tnS)}(\bfhh, \bfR)$ adopts the GeoMFormer architecture~\citep{chen2024geomformer}, which takes the invariant node features $\bfhh$ from the output of the encoder, and the atom coordinates $\bfR$. The output is denoised atom coordinates which are equivariant w.r.t the input coordinates $\bfR$.
These modules are combined to form the energy and structure prediction models following \eqnref{model-composition}.

\subsection{Training Details} \label{appx:training_details}
Our pre-training procedure is executed in two discrete stages. Initially, the model is subjected to training exclusively utilizing the multi-task loss function, $L_{\multitask}$, for a total of 300,000 iterations. Subsequently, in the second stage, we integrate the proposed consistency loss and the force loss, $L_{\textnormal{force}}$, into the training regimen, which then proceeds for an additional 200,000 iterations. All the models are trained with the Adam optimizer~\citep{kingma2015adam} with batch size 256. The learning rate is set to $2\e{-4}$ with a linear warm-up phase in the initial 10,000 steps, which followed by a linear decay schedule thereafter.

The weights of the energy loss and the diffusion denoising loss, \ie, the first and the second terms in \eqnref{loss-multi}, are set to $1.0$ and $0.01$, respectively. The weights of the optimality consistency loss \eqnref{optim-csis} and the score consistency loss \eqnref{score-csis} are set to $0.1$ and $1.0$, respectively.

We employ a sigmoid schedule across 1,000 diffusion time steps for $\beta_t$, with $\beta_0 = 1\e{-4}$ and $\beta_T = 2\e{-2}$.
For the optimality consistency loss, the diffusion time step $\tau$ is sampled uniformly from $[400, 700]$.
For the score consistency loss, the diffusion time step $\tau$ is sampled uniformly from $[5, 300]$, and $k_\tnB \clT$ is set to $0.1 \, \mathrm{eV}$.

The multi-task model is trained on an $8 \times$ Nvidia V100 GPU server for approximately one week. The model with the consistency loss is trained on one $16 \times$ Nvidia V100 GPU server.

\section{Additional Results} \label{appx:add_results}

In this section, we provide additional results under various combinations of settings, and additional metrics and supporting evidence to complement the main results. For clarity, we summarize main result tables in \tabref{res-index} for easier indexing.

\begin{table}[t]
\setlength{\tabcolsep}{3pt}
\centering
\caption{Index of main result tables in the paper.}
\label{tab:res-index}
\footnotesize
\begin{tabular}{lccccccc}
\toprule
Training settings & \tabincell{c}{Test \\ RMSD} & \tabincell{c}{Standard \\ deviation} & \tabincell{c}{Paired \\ t-test} & Coverage & \tabincell{c}{Validation \\ (structure)} & \tabincell{c}{Validation \\ (energy)} & EGap \\
\midrule
PM6 & \tabref{vanilla_test} & \tabref{std_pre} & \tabref{t-test} & \tabref{cov_pretrain} & \tabref{validation-test} & \tabref{energy_valid} & \tabref{egap} \\
PM6 w/ force & \tabref{pretrain-spice-and-pm6-force} & \tabref{std_pre} & \tabref{t-test} & \tabref{cov_pretrain} & \tabref{validation-test} & \tabref{energy_valid} & \tabref{egap} \\
PM6 \phantom{w/ force} + finetuning & \tabref{finetune_test} & \tabref{std_fine} & \tabref{t-test} & \tabref{cov_finetune} & N/A & N/A & N/A \\
PM6 w/ force + finetuning & \tabref{finetune-spice-and-pm6-force} & \tabref{std_fine} & - & \tabref{cov_finetune} & N/A & N/A & N/A \\
\bottomrule
\end{tabular}
\end{table}

\subsection{Additional Fine-Tuning Results} \label{appx:finetune_force}
This section demonstrate more fine-tuning results listed in the main text.
Besides the fine-tuned PM6 dataset pre-trained mode, we performed fine-tuning experiments on the model pre-trained with the SPICE force and PM6 subset force datasets, as well. The results are presented in \tabref{finetune-spice-and-pm6-force}.
Similar to the case in \tabref{finetune_test}, we can again observe that while finetuning improves structure prediction accuracy in all settings (compared to \tabref{pretrain-spice-and-pm6-force}), pre-training the model with consistency loss still enhances the accuracy universally.

\begin{table}[h]
\setlength{\tabcolsep}{5pt}
\centering
\caption{Test RMSD (\angstrom; lower is better) \emph{after finetuning} for structure prediction pre-trained by multi-task learning and consistency learning on the PM6 dataset \emph{with additional SPICE force data or PM6 subset force data}.}
\label{tab:finetune-spice-and-pm6-force}
\begin{tabular}{r|rcccccccc}
\toprule
\multirow{3}{*}{\makecell[cr]{\\ (Pre-)Training \\ Set}}
& Test Set & \multicolumn{4}{c}{PCQ} & \multicolumn{4}{c}{QM9} \\ \cmidrule(lr){3-6} \cmidrule(lr){7-10}
& Generated by & \multicolumn{2}{c}{Denoising} & \multicolumn{2}{c}{DDIM} & \multicolumn{2}{c}{Denoising} & \multicolumn{2}{c}{DDIM} \\
\cmidrule(lr){3-4} \cmidrule(lr){5-6} \cmidrule(lr){7-8} \cmidrule(lr){9-10}
& Struct. Stat. & Mean & Min & Mean & Min & Mean & Min & Mean &  Min  \\ \midrule
\multirow{2}{*}{\makecell[cr]{PM6 with \\ SPICE force}}
& Multi-Task & 1.161 & 0.618 & 0.930 & 0.219 & 0.855 & 0.444 & 0.505 & 0.081 \\
& Consistency & \textbf{1.132} & \textbf{0.581} & \textbf{0.916} & \textbf{0.215} & \textbf{0.832} & \textbf{0.418} & \textbf{0.492} & \textbf{0.073} \\
\midrule
\multirow{2}{*}{\makecell[cr]{PM6 with \\ subset force}}
& Multi-Task & 1.143 & 0.603 & 0.927 & 0.224 & 0.855 & 0.441 & 0.497 & 0.080 \\
& Consistency & \textbf{1.099} & \textbf{0.542} & \textbf{0.914} & \textbf{0.215} & \textbf{0.822} & \textbf{0.419} & \textbf{0.490} & \textbf{0.076} \\ \bottomrule
\end{tabular}
\end{table}

\subsection{Coverage Evaluation for Structure Generation} \label{appx:coverage}
Following the common practice in structure-generation literature, we also test the coverage of the ground-truth structure over model-generated structures. This metric is defined as:
\begin{align}
  \mathrm{COV}(\bbS_\text{gen}, \bfR_\texteq) = \frac{1}{|\bbS_\text{gen}|} \lrvert{\{\hat{\bfR} \in \bbS_\text{gen} \mid \mathrm{RMSD}(\bfR_\texteq, \hat{\bfR}) < \delta \}},
  \label{eqn:def-cov}
\end{align}
where $\bbS_\text{gen}$ denotes the set of generated structures, $\bfR_\texteq$ denotes the ground-truth equilibrium structure provided from the evaluation dataset, $\delta$ is a threshold parameter, and $\lrvert{\cdot}$ takes the cardinality of a set.
Note that since for the task of equilibrium structure prediction, there is only one ground-truth structure, we only evaluate the so-called precision coverage, since the recall coverage (by switching the roles of $\bbS_\text{gen}$ and $\bfR_\texteq$ in the definition \eqnref{def-cov}) evaluates to 1 in all cases.
Here, the RMSD is evaluated after alignment of the two structures by Kabsch algorithm~\citep{kabsch1976solution}.
We choose $\delta$ as 0.9 and 1.25 for QM9 and PCQ dataset. Same as the RMSD test, we use the same test molecules and sample 200 structure for each molecule. The coverage (COV) of the PM6 dataset model prediction and including force dataset model prediction are shown in \tabref{cov_pretrain}. The results present  consistent results as the RMSD. After adding the consistency loss, all have improvement over the multi-task setting. Besides, incorporating the force data can also improve the accuracy.

\begin{table}[h]
\setlength{\tabcolsep}{3.5pt}
\centering
\caption{Test coverage (higher is better) of structure prediction by multi-task learning and consistency learning on the PM6 dataset, and together with additional SPICE force data or PM6 subset force data.}
\label{tab:cov_pretrain}
\begin{tabular}{r|rcccccccc}
\toprule
\multirow{3}{*}{\makecell[cr]{ \\ Training \\ Set}}
& Test Set & \multicolumn{4}{c}{PCQ} & \multicolumn{4}{c}{QM9} \\ \cmidrule(lr){3-6} \cmidrule(lr){7-10}
& Generated by & \multicolumn{2}{c}{Denoising} & \multicolumn{2}{c}{DDIM} & \multicolumn{2}{c}{Denoising} & \multicolumn{2}{c}{DDIM} \\
\cmidrule(lr){3-4} \cmidrule(lr){5-6} \cmidrule(lr){7-8} \cmidrule(lr){9-10}
& Struct. Stat. & Mean & Median & Mean & Median & Mean & Median & Mean &  Median  \\ \midrule
\multirow{2}{*}{\makecell[cr]{PM6}}
&Multi-Task & 0.597 & 0.670 & 0.649 & 0.675 & 0.468 & 0.478 & 0.717 & 0.780 \\
&Consistency & \textbf{0.626} & \textbf{0.695} & \textbf{0.671} & \textbf{0.730} & \textbf{0.604} & \textbf{0.660} & \textbf{0.732} & \textbf{0.825} \\
\midrule
\multirow{2}{*}{\makecell[cr]{PM6 with \\ SPICE force}}
&Multi-Task & 0.614 & 0.660 & 0.645 & 0.685 & 0.550 & 0.580 & 0.715 & 0.765 \\
&Consistency & \textbf{0.644} & \textbf{0.710} & \textbf{0.675} & \textbf{0.725} & \textbf{0.617} & \textbf{0.705} & \textbf{0.740} & \textbf{0.835} \\
\midrule
\multirow{2}{*}{\makecell[cr]{PM6 with \\ subset force}}
&Multi-Task & 0.590 & 0.650 & 0.662 & 0.705 & 0.493 & 0.520 & 0.736 & 0.835 \\
&Consistency & \textbf{0.677} & \textbf{0.775} & \textbf{0.671} & \textbf{0.720} & \textbf{0.643} &\textbf{ 0.690} & \textbf{0.737} & \textbf{0.845} \\
 \bottomrule
\end{tabular}
\end{table}

We also evaluated the coverage on the fine-tuned models. The results are shown in \tabref{cov_finetune}, where we again observe that pre-training with consistency loss improves structure prediction accuracy even after fine-tuning.

\begin{table}[h]
\setlength{\tabcolsep}{3.5pt}
\centering
\caption{Test coverage (higher is better) \emph{after finetuning} for structure prediction pre-trained by multi-task learning and consistency learning on the PM6 dataset, and together with additional SPICE force data or PM6 subset force data.}
\label{tab:cov_finetune}
\begin{tabular}{r|rcccccccc}
\toprule
\multirow{3}{*}{\makecell[cr]{ \\ (Pre-)Training \\ Set}}
& Test Set & \multicolumn{4}{c}{PCQ} & \multicolumn{4}{c}{QM9} \\ \cmidrule(lr){3-6} \cmidrule(lr){7-10}
& Generated by & \multicolumn{2}{c}{Denoising} & \multicolumn{2}{c}{DDIM} & \multicolumn{2}{c}{Denoising} & \multicolumn{2}{c}{DDIM} \\
\cmidrule(lr){3-4} \cmidrule(lr){5-6} \cmidrule(lr){7-8} \cmidrule(lr){9-10}
& Struct. Stat. & Mean & Median & Mean & Median & Mean & Median & Mean &  Median  \\ \midrule
\multirow{2}{*}{\makecell[cr]{PM6}}
&Multi-Task & 0.629 & 0.675 & 0.721 & \textbf{0.800} & 0.543 & 0.595 & \textbf{0.790} & 0.960 \\
&Consistency & \textbf{0.632} & \textbf{0.700} & \textbf{0.724} & 0.795 & \textbf{0.622} & \textbf{0.690} & 0.788 & \textbf{0.990} \\
\midrule
\multirow{2}{*}{\makecell[cr]{PM6 with \\ SPICE force}}
&Multi-Task & 0.632 & 0.705 & 0.714 & 0.760 & 0.597 & 0.660 & 0.784 & 0.980 \\
&Consistency & \textbf{0.649} &\textbf{0.725} & \textbf{0.721} & \textbf{0.795} & \textbf{0.631} & \textbf{0.710} &\textbf{ 0.791} & \textbf{0.990} \\
\midrule
\multirow{2}{*}{\makecell[cr]{PM6 with \\ subset force}}
&Multi-Task & 0.643 & 0.705 & 0.719 & \textbf{0.795} & 0.600 & 0.650 & 0.786 & 0.990 \\
&Consistency & \textbf{0.680} & \textbf{0.780} & \textbf{0.720} & 0.794 & \textbf{0.643} & \textbf{0.715} & \textbf{0.789} & \textbf{0.990} \\
 \bottomrule
\end{tabular}
\end{table}

\subsection{Energy Gap Analysis} \label{appx:egap}

Our goal is to predict the equilibrium structure $\bfR_{\text{eq}} = \argmin_{\bfR} E_{\clG}(\bfR)$.  It is desirable for the predicted structure to approximate this state of minimal energy. To quantify the proximity of the predicted structure to the equilibrium state, we introduce the energy gap metric, which is defined as:
\begin{align}
    \mathrm{EGap} = \frac{E_\clG^{(\bfphi,\bftheta_\tnE)}(\bfR_{\text{pred}}) - E_\clG^{(\bfphi,\bftheta_\tnE)}(\bfR_{\text{eq}})}{|E_\clG^{(\bfphi,\bftheta_\tnE)}(\bfR_{\text{eq}})|}.
\end{align}
The $\mathrm{EGap}$ metric serves to evaluate the energy difference between the predicted and equilibrium structures, with a smaller energy gap signifying a more accurate prediction.

To compute this metric, we randomly select 200 molecules from the intersection of PM6 and PCQ dataset, using the structure from the PCQ dataset as $\bfR_{\text{eq}}$. The predicted structure $\bfR_{\text{pred}}$ is generated using the denoising method. The results are presented in \tabref{egap}. We observe that the incorporation of the consistency loss reduces the EGap metric, particularly in cases with additional force labels. These results demonstrate that the consistency loss effectively transfers information from the energy model to the structure model.

\begin{table}[h]
\centering
\caption{Comparison of averaged EGap between structure prediction by multi-task learning and consistency learning. Lower EGap values suggest that the energy of the predicted structure is closer to the theoretical minimum energy.}
\label{tab:egap}
\begin{tabular}{rccc}
\toprule
Training Set & PM6 & \makecell{PM6 with \\ SPICE force} & \makecell{PM6 with \\ PM6 subset force} \\ \midrule
Multi-Task & 0.1278 & 0.0546 & 0.1163 \\
Consistency & \textbf{0.1172} & \textbf{0.0306} & \textbf{0.1013} \\ \bottomrule
\end{tabular}
\end{table}

\subsection{Validation Results} \label{appx:validation}

To make sure that the conclusions drawn from the above results are solid, we further provide validation results for energy and structure.

For the energy, we randomly selected 200 molecules from the intersection of the PM6 dataset (training set) and the PCQ dataset (test set). This choice allows evaluating energy on both PM6 structure and PCQ structure for each molecule, where the former reflects training quality, and the latter reflects the utility for consistency learning of structure prediction.
Although the original PCQ dataset~\citep{hu2021ogb} does not provide energy labels, we noticed that it is curated based on the PubChemQC Project dataset~\citep{nakata2017pubchemqc}, which provides energy labels on the same PCQ structures under the same DFT settings (B3LYP/6-31G*) as the PM6 energy labels.

The results are listed in \tabref{energy_valid} and boxplots of the distributions of energy prediction MAE (eV) evaluated on the PM6 structures are shown in \figref{box_plot_eng_validation}.
We can observe that energy prediction on PM6 structures are reasonably accurate, and it gets even better when consistency training is activated. This indicates that consistency training does not sacrifice energy prediction.
The energy prediction error is larger on PCQ structures, which is not surprising since the model does not see data with PCQ structures in training.
But as we mentioned in the main paper, better generalization can be expected if force data are available, which enriches the information on the energy landscape hence improving the generalization on unseen structures. From the table, we can observe that the energy prediction accuracy on PCQ structures is indeed improved.
We also notice that including force data in training even improves energy prediction accuracy on PM6 structures, for which a possible explanation is that introducing additional relevant prediction tasks could help the model learn a more comprehensive representation of the input molecular structure, which in turn helps other tasks.

\begin{table}[h]
\centering
\caption{Validation MAE (eV) of energy prediction trained by multi-task learning and consistency learning. Validation molecules are randomly selected from the intersection of PM6 and PCQ, and results on both PM6 structures and PCQ structures of the molecules are shown.}
\label{tab:energy_valid}
\begin{tabular}{r|rccc}
\toprule
\makecell[r]{Validation \\ structures from} & Training Set & PM6 & \makecell{PM6 with \\ SPICE force} & \makecell{PM6 with \\ PM6 subset force} \\
\midrule
\multirow{2}{*}{PM6 structures} &
Multi-Task & 96.08 & 83.15 & 79.00 \\
& Consistency & \textbf{88.58 }& \textbf{78.56} & \textbf{72.61} \\
\midrule
\multirow{2}{*}{PCQ structures} &
Multi-Task & 136.66 & 135.93 & 117.22 \\
& Consistency & \textbf{127.06} & \textbf{110.80} & \textbf{108.71}  \\ \bottomrule
\end{tabular}
\end{table}

\begin{figure}[H]
    \centering
    \includegraphics[width=0.5\linewidth]{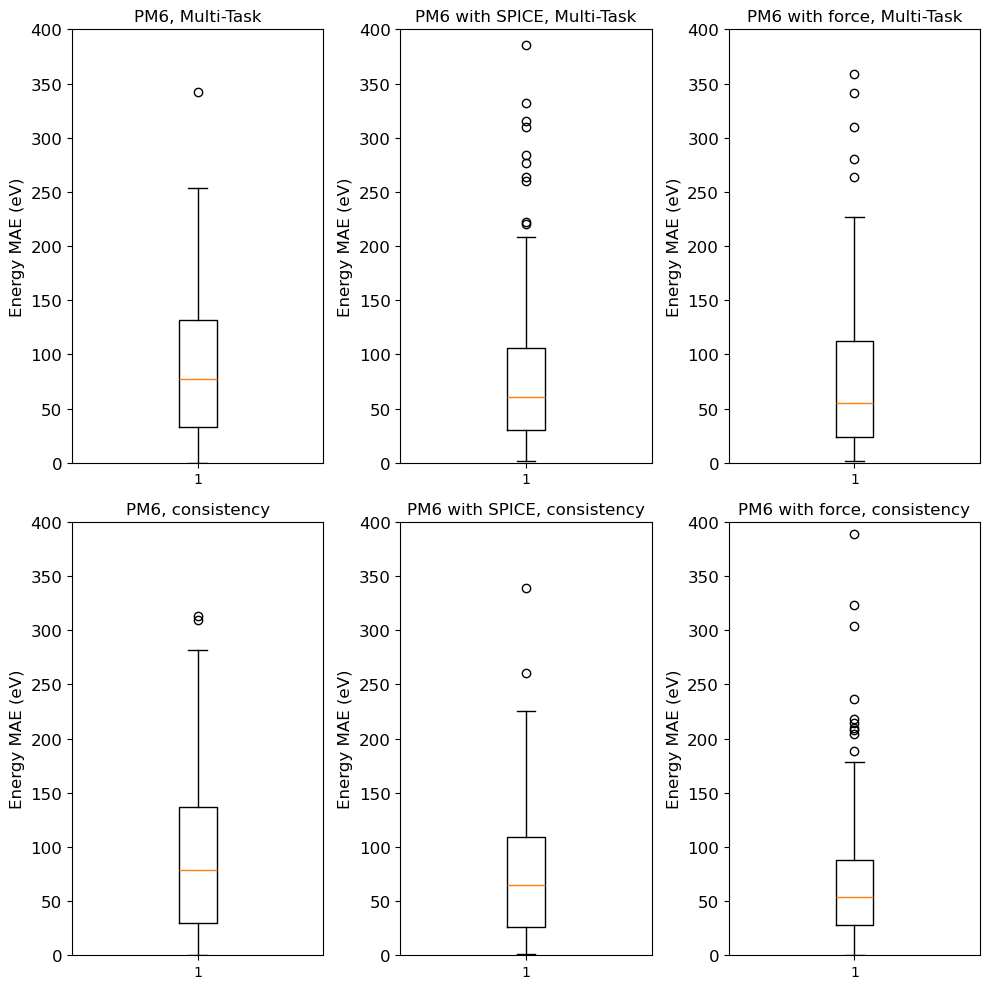}
    \caption{Box plots for the distributions of energy prediction MAE (eV) evaluated on the PM6 structures of randomly selected 200 molecules from the intersection of PM6 and PCQ datasets.
    }
    \label{fig:box_plot_eng_validation}
\end{figure}

For structure prediction, we validate the trained model by evaluating the RMSD of generated structures on 200 randomly selected PM6 molecules.
From the results shown in \tabref{validation-test}, we can see that the models trained by both multi-task learning and consistency learning achieve reasonable predictions.

\begin{table}[h]
\setlength{\tabcolsep}{3pt}
\centering
\caption{Validation RMSD (\angstrom) evaluated on the PM6 validation set of structure prediction trained by multi-task learning and consistency learning on the PM6 dataset, and together with additional SPICE force data or PM6 subset force data.
Predicted structures are generated by the DDIM method.}
\label{tab:validation-test}
\begin{tabular}{rcccccc}
\toprule
Training Set & \multicolumn{2}{c}{PM6} & \multicolumn{2}{c}{\makecell{PM6 with \\ SPICE force}} & \multicolumn{2}{c}{\makecell{PM6 with \\ PM6 subset force}}  \\
\cmidrule(lr){2-3} \cmidrule(lr){4-5} \cmidrule(lr){6-7}
Struct. Stat. & Mean  & Min  & Mean  & Min  & Mean  & Min    \\ \midrule
Multi-Task & 0.970 & 0.287 & 0.983 & 0.314 & 0.953 & 0.267  \\
Consistency & 0.931 & 0.280 & 0.944 & 0.286 & 0.958 & 0.290  \\
\bottomrule
\end{tabular}
\end{table}

\subsection{Standard Deviations and Statistical Significance} \label{appx:std}
To evaluate the significance of our comparison results, for each experimental setup, we repeated five independent runs using different random seeds (666666, 666667, 666668, 666669, 666670). The reported results in the main paper are the means of the five repeats, while the standard deviations are listed in \tabref{std_pre} and \tabref{std_fine} here, which cover the results for the pre-trained models and for fine-tuned models, respectively. We can see that all the standard deviations are around 0.003~$\angstrom$, which is clearly lower than the improvements by consistency training, hence indicating the significance of the effectiveness of the proposed methods.

To more seriously assess the statistical significance, we also conducted a paired t-test for each comparison. The p-values are listed in \tabref{t-test}, where each number represents the probability of the observed results under the null hypothesis that the means by multi-task learning and consistency learning are the same, hence the lower value the more significant that consistency learning outperforms multi-task learning.
Nearly all p-values are well below the 0.05 significance threshold, indicating that the improvement by consistency learning is significant. %
Exceptional cases almost all correspond to the setting where the result in each repeat is the ``Min''imum RMSD over 200 generated samples, in which case multi-task learning may also has a chance to hit the target structure. %

\begin{table}[h]
\setlength{\tabcolsep}{3pt}
\centering
\caption{Standard deviations for the test RMSD (\angstrom) of structure prediction by multi-task learning and consistency learning on the PM6 dataset (corresponding to \tabref{vanilla_test}), and together with additional SPICE force data or PM6 subset force data (corresponding to \tabref{pretrain-spice-and-pm6-force}).}
\label{tab:std_pre}
\begin{tabular}{r|rcccccccc}
\toprule
\multirow{3}{*}{\makecell[cr]{\\ Training \\ Set}}
& Test Set & \multicolumn{4}{c}{PCQ} & \multicolumn{4}{c}{QM9} \\ \cmidrule(lr){3-6} \cmidrule(lr){7-10}
& Generated by & \multicolumn{2}{c}{Denoising} & \multicolumn{2}{c}{DDIM} & \multicolumn{2}{c}{Denoising} & \multicolumn{2}{c}{DDIM} \\
\cmidrule(lr){3-4} \cmidrule(lr){5-6} \cmidrule(lr){7-8} \cmidrule(lr){9-10}
& Struct. Stat. & Mean & Min & Mean & Min & Mean & Min & Mean &  Min  \\
\midrule
\multirow{2}{*}{\makecell[cr]{PM6}}
&Multi-Task & 0.0013 & 0.0045 & 0.0019 & 0.0059 & 0.0004 & 0.0026 & 0.0017 & 0.0015 \\
&Consistency & 0.0021 & 0.0021 & 0.0016 & 0.0053 & 0.0018 & 0.0022 & 0.0080 & 0.0016 \\
\midrule
\multirow{2}{*}{\makecell[cr]{PM6 with \\ SPICE force}}
&Multi-Task & 0.0021 & 0.0042 & 0.0020 & 0.0067 & 0.0013 & 0.0040 & 0.0084 & 0.0024 \\
&Consistency & 0.0026 & 0.0024 & 0.0017 & 0.0030 & 0.0019 & 0.0034 & 0.0025 & 0.0030 \\
\midrule
\multirow{2}{*}{\makecell[cr]{PM6 with \\ subset force}}
&Multi-Task & 0.0113 & 0.0048 & 0.0044 & 0.0127 & 0.0015 & 0.0134 & 0.0020 & 0.0025 \\
&Consistency & 0.0027 & 0.0040 & 0.0023 & 0.0043 & 0.0011 & 0.0030 & 0.0033 & 0.0030 \\ \bottomrule
\end{tabular}
\end{table}

\begin{table}[h]
\setlength{\tabcolsep}{3pt}
\centering
\caption{Standard deviations for the test RMSD (\angstrom) \emph{after finetuning} for structure prediction pre-trained by multi-task learning and consistency learning on the PM6 dataset (corresponding to \tabref{finetune_test}), and together with additional SPICE force data or PM6 subset force data (corresponding to \tabref{finetune-spice-and-pm6-force}).}
\label{tab:std_fine}
\begin{tabular}{r|rcccccccc}
\toprule
\multirow{3}{*}{\makecell[cr]{\\ (Pre-)Training \\ Set}}
& Test Set & \multicolumn{4}{c}{PCQ} & \multicolumn{4}{c}{QM9} \\ \cmidrule(lr){3-6} \cmidrule(lr){7-10}
& Generated by & \multicolumn{2}{c}{Denoising} & \multicolumn{2}{c}{DDIM} & \multicolumn{2}{c}{Denoising} & \multicolumn{2}{c}{DDIM} \\
\cmidrule(lr){3-4} \cmidrule(lr){5-6} \cmidrule(lr){7-8} \cmidrule(lr){9-10}
& Struct. Stat. & Mean & Min & Mean & Min & Mean & Min & Mean &  Min  \\ \midrule
\multirow{2}{*}{\makecell[cr]{PM6}}
&Multi-Task & 0.0021 & 0.0059 & 0.0025 & 0.0036 & 0.0008 & 0.0026 & 0.0011 & 0.0044 \\
&Consistency & 0.0024 & 0.0064 & 0.0020 & 0.0033 & 0.0011 & 0.0028 & 0.0022 & 0.0023 \\
\midrule
\multirow{2}{*}{\makecell[cr]{PM6 with \\ SPICE force}}
&Multi-Task & 0.0023 & 0.0053 & 0.0012 & 0.0051 & 0.0016 & 0.0038 & 0.0019 & 0.0033 \\
&Consistency & 0.0023 & 0.0047 & 0.0030 & 0.0013 & 0.0018 & 0.0030 & 0.0027 & 0.0015 \\
\midrule
\multirow{2}{*}{\makecell[cr]{PM6 with \\ subset force}}
&Multi-Task & 0.0022 & 0.0042 & 0.0016 & 0.0053 & 0.0013 & 0.0036 & 0.0019 & 0.0014 \\
&Consistency & 0.0019 & 0.0038 & 0.0019 & 0.0072 & 0.0016 & 0.0068 & 0.0016 & 0.0023 \\ \bottomrule
\end{tabular}
\end{table}

\begin{table}[h]
\setlength{\tabcolsep}{1.5pt}
\centering
\caption{Paired t-test p-values on structure prediction RMSD means over 5 repeats (standard deviations are shown in Tables~\ref{tab:std_pre} and~\ref{tab:std_fine}) corresponding to the results in \tabref{vanilla_test} (row 1, pre-training on PM6), \tabref{pretrain-spice-and-pm6-force} (rows 2 and 3, pre-training on PM6 together with force labels), and \tabref{finetune_test} (row 4, pre-training on PM6 then finetuning).
Values lower than the 0.05 significance threshold are shown in bold.
}
\label{tab:t-test}
\renewcommand{\e}[1]{{\scriptstyle \!\times\! 10^{#1}}}
\footnotesize
\begin{tabular}{rcccccccc}
\toprule
Test Set & \multicolumn{4}{c}{PCQ} & \multicolumn{4}{c}{QM9} \\ \cmidrule(lr){2-5} \cmidrule(lr){6-9}
Generated by & \multicolumn{2}{c}{Denoising} & \multicolumn{2}{c}{DDIM} & \multicolumn{2}{c}{Denoising} & \multicolumn{2}{c}{DDIM} \\
\cmidrule(lr){2-3} \cmidrule(lr){4-5} \cmidrule(lr){6-7} \cmidrule(lr){8-9}
Struct. Stat. & Mean & Min & Mean & Min & Mean & Min & Mean & Min \\ \midrule
PM6  & $\bm{7.8 \e{-4}}$ & $\bm{5.6 \e{-3}}$ & $\bm{3.8 \e{-7}}$ & $\bm{8.9 \e{-4}}$ & $\bm{4.5 \e{-7}}$ & $\bm{8.2 \e{-7}}$ & $\bm{1.1 \e{-3}}$ & $\bm{7.4 \e{-4}}$ \\
\midrule
PM6 w/ SPICE force  & $\bm{2.4 \e{-5}}$ & $\bm{1.0 \e{-5}}$ & $\bm{6.0 \e{-4}}$ & $\bm{3.1 \e{-5}}$ & $\bm{9.9 \e{-8}}$ & $4.5 \e{-1}$ & $\bm{2.3 \e{-2}}$ & $\bm{1.5 \e{-4}}$ \\
\midrule
PM6 w/ subset force  & $\bm{4.5 \e{-3}}$ & $\bm{2.4 \e{-5}}$ & $\bm{2.9 \e{-3}}$ & $\bm{2.4 \e{-2}}$ & $\bm{6.5 \e{-7}}$ & $\bm{1.2 \e{-4}}$ & $5.1 \e{-2}$ & $7.5 \e{-1}$ \\
\midrule
PM6 + finetuning  & $\bm{4.2 \e{-3}}$ & $2.4 \e{-1}$ & $\bm{2.4 \e{-3}}$ & $1.7 \e{-1}$ & $\bm{1.2 \e{-7}}$ & $\bm{2.5 \e{-6}}$ & $\bm{1.0 \e{-4}}$ & $\bm{8.5 \e{-3}}$ \\
\bottomrule
\end{tabular}
\end{table}

\newpage
\subsection{Evaluation on Dissimilar Molecules} \label{appx:dissimilar}
To further consolidate the effectiveness of our method, we give a closer look into the influence of the similarity of the test dataset to the training dataset.
Recall that we have excluded identical molecules appearing in the training dataset from the test dataset, but there remains the possibility that the test dataset may still contain similar molecules as those in the training dataset. For this, we first investigate the Tanimoto similarity~\citep{bajusz2015tanimoto} between the training and test datasets. This similarity can be computed from the molecular graphs of two molecules, while considering structural similarity between the two molecules. We plot in \figref{tanimoto_sim} where each column visualizes the count of PCQ test molecules that have a certain portion of similar (Tanimoto similarity > 0.7) molecules in the PM6 training dataset.
We can observe that most (almost all) of the test molecules only has less than $1\e{-7}$ ($2.5\e{-7}$) similar molecules in PM6. This indicates that excluding identical molecules appearing in the training dataset from the test dataset already makes the test dataset sufficiently dissimilar to the training dataset, hence the presented results are valid prediction evaluations that are not close to ``memorizing the training molecules''. %

\begin{figure}[H]
    \centering
    \vspace{4pt}
    \includegraphics[width=0.6\linewidth]{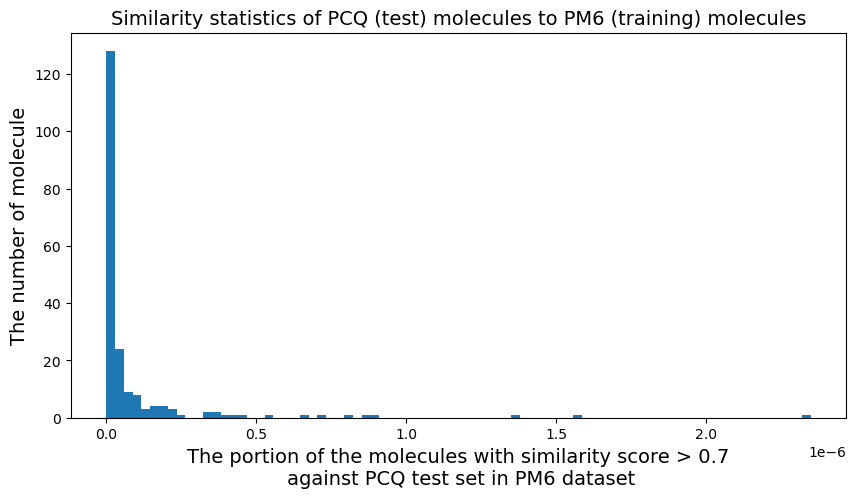}
    \caption{Histogram showing the distribution of the portion of similar (Tanimoto similarity > 0.7) molecules in PM6 (the training dataset) over the 200 PCQ test molecules. Note that the x-axis is scaled by $1\e{-6}$.
    }
    \label{fig:tanimoto_sim}
\end{figure}

\vspace{4pt}
For a completely sanitized evaluation on dissimilar molecules, we provide the results evaluated on PCQ test molecules that do not have any similar (Tanimoto similarity > 0.7) molecules in the PM6 (training) dataset. Such molecules make up 49\% of the total 200 PCQ test molecules. \tabref{test_disimilar} shows the results in both RMSD and Coverage using the denoising generation method, corresponding to the settings in Tables~\ref{tab:vanilla_test} and~\ref{tab:cov_pretrain} (row block 1) (pre-training on PM6), and Tables~\ref{tab:pretrain-spice-and-pm6-force} and~\ref{tab:cov_pretrain} (row blocks 2 and 3) (pre-training on PM6 together with force labels).
The results confirmed the advantage of consistency learning on dissimilar molecules that cannot rely on memorizing the training dataset, which further consolidates the conclusion.

\begin{table}[h]
\setlength{\tabcolsep}{3pt}
\centering
\caption{Test RMSD (\angstrom; lower is better) and coverage (higher is better) \emph{on dissimilar molecules} for structure prediction pre-trained by multi-task learning and consistency learning on the PM6 dataset, and together with additional SPICE force data or PM6 subset force data.}
\label{tab:test_disimilar}
\begin{tabular}{r|rcccc}
\toprule
\multirow{2}{*}{\makecell[cr]{\\ Training Set}}
& \multirow{2}{*}{\makecell[cr]{\\ Method}} & \multicolumn{2}{c}{RMSD (\angstrom)} & \multicolumn{2}{c}{Cov} \\
\cmidrule(lr){3-4} \cmidrule(lr){5-6}
&              & Mean & Min  & Mean & Median \\
\midrule
\multirow{2}{*}{\makecell[cr]{PM6 \\ (\cf Tables~\ref{tab:vanilla_test} or~\ref{tab:cov_pretrain} (row block 1))}}
& Multi-Task   & 1.175 & 0.642 & 0.613 & 0.675 \\
& Consistency  & \textbf{1.135} & \textbf{0.625} & \textbf{0.644} & \textbf{0.745} \\
\midrule
\multirow{2}{*}{\makecell[cr]{PM6 w/ SPICE force \\ (\cf Tables~\ref{tab:pretrain-spice-and-pm6-force} (row block 1) or~\ref{tab:cov_pretrain} (row block 2))}}
& Multi-Task   & 1.136 & 0.609 & 0.639 & 0.735 \\
& Consistency  & \textbf{1.121} & \textbf{0.579} & \textbf{0.672} & \textbf{0.790} \\
\midrule
\multirow{2}{*}{\makecell[cr]{PM6 w/ subset force \\ (\cf Tables~\ref{tab:pretrain-spice-and-pm6-force} (row block 2) or~\ref{tab:cov_pretrain} (row block 3))}}
& Multi-Task   & 1.174 & 0.653 & 0.612 & 0.660 \\
& Consistency  & \textbf{1.099} & \textbf{0.616} & \textbf{0.697} & \textbf{0.830} \\
\bottomrule
\end{tabular}
\end{table}

\end{document}